\documentclass[letterpaper, 11pt]{article}

\usepackage[margin=1.2in]{geometry}
\usepackage{natbib}
\usepackage[utf8]{inputenc} 
\usepackage[T1]{fontenc}    
\usepackage{hyperref}       
\usepackage{url}            
\usepackage{booktabs}       
\usepackage{amsfonts}       
\usepackage{nicefrac}       
\usepackage{microtype}      

\usepackage{mathtools}
\usepackage{verbatim}

\usepackage{amsmath,amsthm,amssymb}
\usepackage{bbm}
\usepackage{subcaption}
\usepackage{appendix}

\usepackage{algorithm}
\usepackage[noend]{algorithmic}

\newcommand{\bE}{\mathbb{E}}

\newcommand{\eps}{\varepsilon}

\newcommand{\Cov}{\mathrm{Cov}}

\newcommand{\E}{\bE}      

\newcommand{\KL}{\mathop{\bf KL\/}}
\newcommand{\SKL}{\mathop{\bf SKL\/}}
\newcommand{\TV}{{\bf TV}}

\newcommand{\cov}{\mathsf{Cov}}
\newcommand{\avg}{\mathsf{Avg}}

\newcommand{\bone}{\mathbbm{1}}

\DeclareMathOperator*{\argmax}{arg\,max}

\usepackage{xcolor}
\newcommand{\fnote}[1]{\textcolor{red}{Fred: #1}}

\newcommand{\ignore}[1]{}

\allowdisplaybreaks
\newtheorem{theorem}{Theorem}
\newtheorem{corollary}{Corollary}
\newtheorem{lemma}{Lemma}
\newtheorem{fact}{Fact}
\newtheorem{defn}{Definition}
\newtheorem{remark}{Remark}
\newtheorem{example}{Example}
\newtheorem{assumption}{Assumption}

\title{From Boltzmann Machines to Neural Networks and Back Again}

\author{%
  Surbhi Goel \\
  Microsoft Research NYC\\
  \texttt{surbgoel@microsoft.com}
  \and
  Adam Klivans\\
  University of Texas at Austin\\
  \texttt{klivans@cs.utexas.edu}
  \and 
  Frederic Koehler\\
  MIT\\
  \texttt{fkoehler@mit.edu}
}
\date{}
\begin{document}

\maketitle

\begin{abstract}
Graphical models are powerful tools for modeling high-dimensional data, but learning graphical models in the presence of latent variables is well-known to be difficult. In this work we give new results for learning Restricted Boltzmann Machines, probably the most well-studied class of latent variable models. Our results are based on new connections to learning two-layer neural networks under $\ell_{\infty}$ bounded input; for both problems, we give nearly optimal results under the conjectured hardness of sparse parity with noise. Using the connection between RBMs and feedforward networks, we also initiate the theoretical study of {\em supervised RBMs} \citep{hinton2012practical}, a version of neural-network learning that couples distributional assumptions induced from the underlying graphical model with the architecture of the unknown function class.  We then give an algorithm for learning a natural class of supervised RBMs with better runtime than what is possible for its related class of networks without distributional assumptions.
\end{abstract}

\section{Introduction}
Graphical models are a powerful framework for modelling high-dimensional distributions in a way that is interpretable and enables sophisticated forms of inference and reasoning. They are extensively used in a variety of disciplines including the natural and social sciences where they have been used to model the structure of gene regulatory networks, of connectivity in the brain, and the flocking behavior of birds \citep{bialek2012statistical}. In many contexts, the structure of interactions between different observed variables is unknown a priori and the goal is to infer this structure in a sample-efficient way from data. There has been decades of research on various formulations of this problem, both theoretically and empirically: for example, provable algorithms have been developed for learning tree-structured graphical models \citep{ChowL}, for learning models on graphs of bounded tree-width \citep{karger2001learning}, for learning Ising models on general graphs of bounded degree \citep{BreslerMosselSly,Bresler,Vuffray,KlivansM} and in a variety of other contexts like Gaussian graphical models (e.g. \citep{meinshausen2006high}). For the most part, the main interest has been on learning under the assumption that the underlying model is sparse. Sparsity is a natural assumption since many applications are in a sample-starved regime where the learning problem is information-theoretically impossible without sparsity. Sparse models are generally considered to be more interpretable than their dense counterparts since they satisfy large numbers of conditional independence relations.

A major challenge in probabilistic inference from data is the presence of latent or confounding variables which are unobserved and may create complicated higher-order dependencies between the observed variables. Specifically in the context of learning undirected graphical models, it is well known that even if the underlying graphical model is well-behaved, if only a subset of the variables are observed then the resulting marginal distribution can still be extremely complicated, e.g. simulating the uniform distribution over satisfying assignments of an arbitrary circuit \citep{BogdanovMV}, which makes the learning problem computationally intractable. On the other hand, under certain assumptions we know that learning graphical models with latent variables can be both computationally and statistical tractable; for example, the setting of tree-structured models with latent variables has been extensively studied in the context of phylogenetic reconstruction, see e.g. \citep{felsenstein2004inferring,daskalakis2006optimal}. However, in non-tree-structured models there are comparatively few positive results for recovering latent variable models in a computationally efficient fashion. One of the few exceptions is in the Gaussian case, where \citep{ChandrasekaranPW} gave a positive result; this setting is very special, as latent variable GGMs \emph{do not} have higher-order interactions, but in fact are equivalent to GGMs with cliques.

In this work, we will focus on a latent variable model popularized in the neural network literature known as the \emph{Restricted Boltzmann Machine} (RBM) (see e.g. \citep{hinton2012practical,goodfellow2016deep}) which has been applied to problems such as dimensionality reduction and collaborative filtering \citep{RBMdimension,RBMclassification,RBMfeature}. It is also perhaps the most canonical version of an Ising model with latent variables.
The RBM describes a joint distribution over observed random variables $X$ valued in $\{\pm 1\}^{n_1}$ and latent variables $H$ valued in $\{ \pm 1\}^{n_2}$
\begin{align*} \Pr(X = x,H = h) \propto \exp\left(\langle x, Wh \rangle + \langle b^{(1)}, x \rangle + \langle b^{(2)}, h \rangle \right) \end{align*}
where the \emph{weight matrix} $W$ is an arbitrary $n_1 \times n_2$ matrix and \emph{external fields}/\emph{biases} $b^{(1)} \in \mathbb{R}^{n_1}$ and $b^{(2)} \in \mathbb{R}^{n_2}$ are arbitrary, and $X$ is referred to as the vector of \emph{visible unit} activations and $H$ the vector of \emph{hidden unit} activations. In the learning problem, we are given access to i.i.d. samples of $X$ but do not get to observe $H$. It is not hard to see that in the special case where the hidden nodes are constrained to have degree 2, the class of marginal distributions on $X$ induced by RBMs is exactly the class of Ising models (pairwise binary graphical models), so the general RBM can be thought of as a natural generalization of fully-observed Ising models, for which the learning problem is well-understood. We also note that the parameters of the RBM are not identifiable even given an infinite number of samples, so our goal for learning the RBM is generally speaking to learn the distribution or related structural properties (e.g. the Markov blankets of the nodes in $X$).

In recent work \citep{bresler2019learning,goel2019learning}, the first provable algorithms were developed for learning RBMs, under the assumptions that the model is (1) sparse and (2) ferromagnetic. On the other hand, it was shown in \citep{bresler2019learning} that learning general sparse RBMs is computationally intractable in general, because the conjecturally hard problem of learning a \emph{sparse parity with noise} can be embedded into a sparse RBM with a constant number of hidden units. The assumption of ferromagneticity (that variables are only positively correlated, not negatively correlated) rules out this example and plays a crucial role in the analysis of these works. Without ferromagneticity, viewing $X$ as the observed distribution of a general Markov Random Field allows for using prior work \citep{KlivansM} to give learning algorithms with runtime $n^{O(d_H)}$ where $d_H$ is the maximum degree of a hidden node. This matches the lower bound of learning sparse parity with noise mentioned previously.

To summarize, the best previous results for learning RBMs either (1) make the assumption of ferromagneticity which makes building sparse parities impossible or (2) ignore all of the structure of the RBM except the max hidden degree, and pay the price of a $n^{\Theta(d_H)}$ runtime. This leaves open the question of developing algorithms whose runtime depends on some natural notion of a \emph{complexity} measures of the RBM.

In this paper, we design an algorithm that is adaptive to a {\em norm} based complexity measure of the RBM, and often outperforms approach (2) above significantly, while not eliminating the possibility of negative correlation completely as in (1). The key idea of our approach is to develop a novel connection between learning RBMs and their historical relative, feedforward neural networks. This connection allows us to establish new results for learning RBMs via proving new results about learning feedforward neural networks (Section \ref{sec:learningrbmff}). 

Our connection also validates the idea of a so-called {\em supervised RBMs} as a natural distributional setting for classification with feedforward networks. Supervised RBMs, proposed by Hinton \citep{hinton2012practical}, treat one visible unit of the RBM as the label and the other visible units as the input to the classifier. This allows us to use the
 connection in the ``reverse'' direction --- using natural structural assumptions on the RBM (like ferromagneticity) to give better results for solving supervised prediction tasks in an interesting distributional setting. Along these lines, we show that an assumption related to ferromagneticity, but allowing for some amount of negative correlation in the RBM, allows us to learn the induced feedforward network faster than would be possible without distributional assumptions  (Section \ref{sec:srbm}). Lastly, we present an experimental evaluation of our "supervised RBM" algorithm on MNIST and FashionMNIST to highlight the applicability of our techniques in practice (Section \ref{sec:exp}).

\section{Learning RBMs via New Results for Feedforward Networks}\label{sec:learningrbmff}
\paragraph{Relationship between RBMs and Feedforward Networks}
Our first result characterizes the relationship between RBMs and Feedforward networks. We show that there is a natural self-supervised prediction task in RBMs, of predicting the spin at node $i$ given all other observed nodes, for which the Bayes-optimal predictor is \emph{exactly given} by a two-layer feedforward network with a special family of $\tanh$-like activations.
\begin{theorem}\label{thm:node-prediction}
For any visible unit $i$ in an arbitrary RBM,
\begin{equation}\label{eqn:node-prediction} 
\E[X_i | X_{\sim i}] = \tanh\left(b^{(1)}_i + \sum_j \tanh(W_{ij}) f_{\beta_{ij}}\Big(b^{(2)}_j + \sum_{k \ne i} W_{kj} X_k\Big)\right) 
\end{equation}
where $\beta_{ij} = |\tanh(W_{ij})|$ and $f_{\beta}(x) := \frac{1}{\beta} \tanh^{-1}(\beta \tanh(x))$.
\end{theorem}
\begin{proof}
Observe that the conditional distribution of $(X_i,H)$ given $X_{\sim i} = x_{\sim i}$ is given by
\begin{equation}\label{eqn:cond-joint-law}
\Pr(X_i = x_i, H = h | X_{\sim i} = x_{\sim i}) \propto \exp\left(x_i(b^{(1)}_i + \sum_j W_{ij} h_j) + \langle W_{\sim i}^t x_{\sim i} + b^{(2)}, h \rangle\right)
\end{equation}
where $W_{\sim i}$ denotes the $(n_V - 1) \times n_H$ dimensional matrix given by deleting row $i$. Since the only quadratic terms left in the potential are between the remaining visible unit $X_i$ and the hidden units $h_j$, this conditional distribution is exactly an Ising model on a star graph, i.e. a tree of depth $1$ with root node corresponding to $X_i$. For all tree-structured graphical models, the conditional distribution of the root given the leaves can be computed exactly by Belief Propagation (see e.g. \citep{mezard2009information,pearl2014probabilistic}); in the case of Ising models it's known the general BP formula can be written with hyperbolic functions as above\footnote{For the readers convenience, we include a self-contained derivation of \eqref{eqn:node-prediction} from \eqref{eqn:cond-joint-law} in Appendix~\ref{apdx:conditional-law}.}.
\end{proof}
\begin{remark}
An analogous result can be proved in the more general setting where the spins do not have to be binary; for example in a Potts model version of the RBM where each spin is valued in a set of size $q$, the conditional law of $X_i$ given the others would be given again by a two-layer network where the last layer is a softmax. In this paper we focus on the binary case for simplicity.
\end{remark}
\begin{remark}\label{rmk:f-activation}
The family of activation functions $f_{\beta}(x)$ naturally interpolates between the identity activation ($\beta = 1$ where $f_{\beta}(x) = x$) and $\tanh$ activation at $\beta = 0$, since
\[ \lim_{\beta \to 0} \frac{1}{\beta} \tanh^{-1}(\beta \tanh(x)) = \frac{\partial}{\partial \beta} \tanh^{-1}(\beta \tanh(x))\Big|_{\beta = 0} = \tanh(x). \] 
\end{remark}
The exact structure of this prediction function is crucial in what follows and does not seem to have been known in the RBM literature, though some related ideas have been used to develop better heuristics for performing inference and training in RBMs (see discussion in Section~\ref{sec:connections}). 

Given this connection, we show that if we can solve the problem of learning such a neural network within sufficiently small error, then we can successfully learn the RBM. This reduces our RBM learning problem to that of learning feedforward neural networks in the setting that the input is bounded in $\ell_{\infty}$ norm.  

\paragraph{Improved Results for Learning Feedforward Networks}
Subsequently, we give results for the feedforward network problem which are
nearly optimal both in the terms of sample complexity (in the regime where $\lambda$ is bounded) and in terms of computational complexity under the hardness of learning sparse parity with noise; some aspects of this result are new even for the well-studied case of learning neural networks with $\tanh$ activations (see Further Discussion).
\begin{theorem}[Informal version of Corollary~\ref{corr:rbm-regression-simple}]\label{thm:regression-intro}
Suppose that $Y$ is a random variable valued in $\{\pm 1\}$, $X$ is a random vector such that $\|X\|_{\infty} \le 1$ almost surely and
\[ \E[Y | X] = \tanh\left(b^{(1)} + \sum_j w_j f_{\beta_j}\Big(b^{(2)}_j + \sum_{k} W_{jk} X_k\Big)\right) \]
where $b^{(1)} \in \mathbb{R}$, $\beta_j \in [0,1]$, $w$ is an arbitrary real vector and $W$ is an arbitrary real matrix. Let $W_j$ denote \emph{column} $j$ of $W$ and suppose $\|W_j\|_1 \le \lambda$ for every $j$ and some $\lambda \ge 2$. Then if we run $\ell_1$-constrained regression on the degree $D$ monomial feature map $\varphi_D(x) \mapsto \left(\prod_{i \in S} X_i\right)_{|S| \le d}$ with appropriate $\ell_1$ constraint, the result $\hat{w}$ satisfies with high probability
\[ \E[\ell(\hat{w} \cdot \varphi_d(X), Y)] \le OPT + \epsilon \]
where $OPT$ is the minimum logistic loss for any measurable function of $X$,
as long as the number of samples $m$ satisfies $m = \Omega((|b^{(1)}|^2 \lambda^{O(D)})\log(2n))$ where $D = O(\lambda \log(\|w\|_1\lambda/\epsilon))$ and the runtime of the algorithm is $poly(n^D)$.
\end{theorem}
We also show, under the standard assumption for hardness of learning sparse parity with noise, the following lower bound which shows 
 that the runtime guarantee in our result is close to tight even in the usual setting of $\tanh$ neural networks ($\beta_j = 0$) --- it is optimal up to $\log \log$ factors in the exponent in its dependence on $\epsilon$ and $\|w\|_1$, and we also show that at least a subexponential dependence (essentially $2^{\sqrt{\lambda}}$) on $\lambda$ is unavoidable (assuming the dependence on other parameters in the statement is fixed, since there are e.g. trivial algorithms that run in time $2^n$).
\begin{theorem}[Informal version of Theorem \ref{thm:net-hardness}]
There exists families of models (one with $\epsilon$ a constant, one with $\|w\|_1$ a constant) where a runtime of $n^{\Omega\left(\frac{\log(\|w\|_1/\epsilon)}{\log\log(\|w\|_1/\epsilon)}\right)}$ is needed for any algorithm to achieve $\epsilon$ error with high probability, regardless of its sample complexity. Even in the case of $\tanh$ activations ($\beta_j = 0$ for all $j$), there exists a sequence of models with $\lambda = \Theta(n \log(n))$ and $\|w\|_1 = O(\sqrt{n})$ which requires runtime $n^{\Omega(\sqrt{\lambda/\log^3(\lambda)} \log \|w\|_1)}$ to achieve error $\epsilon = 0.01$ with high probability.
\end{theorem}
To our knowledge, the fact that $n^{\log(\|w\|_1/\epsilon)/\log\log(\|w\|_1/\epsilon)}$ runtime is required to learn this class even for $\lambda = 1$, and by the above upper bound is tight up to the $\log \log$ term, was not known before even for standard $\tanh$ networks. As far as the dependence on $\lambda$, a similar problem was studied in \citep{shalev2011learning} where they proved the dependence cannot be polynomial using the result of \citep{klivans2009cryptographic} for intersection of halfspaces, based on a different assumption, though our lower bound seems to be somewhat stronger in the present context.

In particular the lower bounds on the runtime show that methods like the kernel trick cannot significantly improve the runtime compared to the simple method of writing out the feature map explicitly used in Theorem~\ref{thm:regression-intro}; however, writing out the feature map lets us use $\ell_1$ regularization\footnote{Interestingly, recent work \citep{woodworth2019kernel} has shown in a special case connections between the implicit bias of gradient descent in feedforward networks and $\ell_1$ regularization in function space.} instead of $\ell_2$ which can give significant sample complexity advantages (e.g. $O(\log n)$ vs $O(n)$ for the usual sparse linear regression setups). 

\paragraph{Structure Learning of RBMs}
As explained above, our reduction based on Theorem~\ref{thm:node-prediction} lets us use the above feedforward network learning result to learn the structure of RBMs. By structure learning, we mean learning the \emph{Markov blanket} of the each visible unit in the marginal distribution of the RBM over visible units, i.e. the minimal set of nodes $S$ such that $X_i$ is conditionally independent of all other $X_j$ conditionally on $X_S$. We will also refer to the Markov blanket as the (two-hop) neighborhood of node $i$.
This is a natural objective as other tasks such as distribution learning are straightforward in sparse models if the Markov blankets are known.
As in the previous work on structure learning in other undirected graphical models (e., we will need some kind of quantitative nondegeneracy condition to guarantee nodes in the Markov blanket of node $i$ are information-theoretically discoverable; it is not hard to see (e.g. using the bounds from \citep{santhanamW}) that if two nodes are neighbors but their interaction is extremely weak then it becomes impossible to distinguish the model from the same model with the edge removed without a very large number of samples. 

In Ising models and in ferromagnetic RBMs, there are simple conditions on the weight matrices which can ensure neighbors are information-theoretically discoverable. 
In a general RBM, there is no natural way to place constraints on the weights of the RBM to ensure this: the issue is that two nodes $X_i$ and $X_j$ can be independent even though they have two neighboring hidden units with non-negligible edge weights, since the effect of those hidden units can exactly cancel out so that $X_i$ and $X_j$ are independent or indistinguishably close to independent (a number of examples are given in \citep{bresler2019learning}). For this reason, we will instead make the following assumption on the behavior of the model itself instead of on its weight matrix:
\begin{defn}
We say that visible nodes $i,j$ are $\eta$-nondegenerate two-hop neighbors if 
\[ I(X_i;X_j | X_{\sim i,j}) = \E[\ell(\E[X_i | X_{\sim \{i,j\}}], X_i)] - \E[\ell(\E[X_i | X_{\sim i}], X_i)] \ge \eta \]
or if the same inequality holds with $i$ and $j$ interchanged. Here $I(X_i;X_j | X_{\sim i,j})$ is the \emph{conditional mutual information} between $X_i$ and $X_j$ conditional on $X_{\sim i,j}$, and the equality follows from Fact~\ref{fact:logistic-loss} in the Appendix and the definition of mutual information in terms of KL \citep{cover2012elements}.
\end{defn}
Information-theoretically, this condition says that nontrivial information is gained about $X_i$ by observing $X_j$, even after we have already observed $X_{\sim i,j}$. The fact that $X_j$ is in the Markov blanket of node $X_i$ exactly means that this quantity is nonzero.
By Pinsker's inequality \citep{cover2012elements}, $\eta$-nondegeneracy is also implied by a lower bound on the partial correlation $\Cov(X_i, X_j | X_{\sim i,j})$. 
\begin{example}\label{example:ising-eta}
It is not hard to see that Ising models are equivalent to the marginal distribution of RBMs with maximum hidden node degree equal to $2$. Consider an Ising model with minimum edge weight $\alpha$ and such that the maximum $\ell_1$-norm into every node is upper bounded by $\lambda$ and the external field is upper bounded by $B$, then $\eta \ge e^{-O(\lambda + B)}/\alpha$, see e.g. \citep{Bresler}.
\end{example}
\begin{example}
In a ferromagnetic RBM with minimum edge weight $\alpha$ and maximum external field $B$, it can be shown that $\eta \ge e^{-O(\lambda_1 + \lambda_2 + B)}/\alpha^2$ (see \citep{bresler2019learning,goel2019learning}).
\end{example}

In order for the RBM to be learnable with a reasonable number of samples (since general RBMs can represent arbitrary distributions), we need to assume it has low complexity in the following sense:
\begin{defn}
We say that an RBM is $(\lambda_1,\lambda_2)$-bounded if for any $i$, $\sum_j |\tanh(W_{ij})| + |b^{(1)}_i| \le \lambda_1$ and the columns of $W$ are bounded in $\ell_1$ norm by $\lambda_2$. 
\end{defn}
Note that $\lambda_1$ and  $\lambda_2$ bound the $\ell_1$ norm into the visible and hidden units, respectively. Based on our upper bounds and lower bounds for the learnability of feedforward networks, it should be less surprising that these parameters play a very different role in the computational learnability of RBMs.
\begin{theorem}[Informal version of Theorem~\ref{thm:rbm-structure-recovery}]\label{thm:rbm-structure-recovery-informal}
Suppose all two-neighbors in a $(\lambda_1,\lambda_2)$-bounded RBM are $\eta$-nondegenerate. Given $m = \Omega(\lambda_2^{O(D)} \log(2n))$ i.i.d. samples from the RBM, where $D = O(\lambda_2 \log(\lambda_1 \lambda_2/\eta))$, we can recover its structure with high probability in time $poly(n^D)$.
\end{theorem}
Based on this result we also give a result for learning the RBM in TV distance under the same assumption: see Theorem~\ref{thm:dist-recovery-rbm}: the sample complexity of this method is essentially the above sample complexity plus $n^2 (1 - \tanh(\lambda_1))^{-d_2}$ where $d_2$ is the maximum 2-hop degree; the $poly(n)$ dependence is required as even learning $n$ bernoullis in TV requires $\Omega(n)$ sample complexity. Our algorithm encodes the distribution as a sparse Markov Random Field, but (if desired) this can easily be converted into a sparse RBM using an algorithm in \citep{bresler2019learning}. Therefore we learn the distribution properly, except that the learned RBM typically has more hidden units than the original RBM (i.e. it is overparameterized).

When interpreting these result, it is crucial not to confuse the $\ell_1$ norm parameters $\lambda_1,\lambda_2$ of visible and hidden units with the maximum degrees of these units. Typically in Ising models, we should think of the weight of a typical edge as \emph{shrinking} as $d$ grows so that units stay near the sensitive region of their activation and the behavior of the model does not become trivial --- this means that $\lambda_1$ and $\lambda_2$ may be much smaller than $d$. For example, probably the most well known sufficient condition for being able to sample in an Ising model (or RBM) is \emph{Dobrushin's uniqueness criterion} which is equivalent to the requirement that $\lambda_1,\lambda_2 \le 1$ and this condition is actually tight for Glauber dynamics to mix quickly in the Ising model on the complete graph (Curie-Weiss Model) \citep{levin2017markov}. We discuss this further in Remark~\ref{rmk:sampling-regimes}; in Dobrushin's uniqueness regime and under some mild nondegeneracy conditions we expect that $\eta = \Omega(1/d^2)$ so the above algorithm has runtime $n^{\log(d)}$, which is an exponential improvement in the exponent compared to the best previously known result ($O(n^d)$ runtime by viewing the RBM as an MRF).

We also give lower bound results showing that the computational complexity of the above algorithm is essentially optimal in terms of $\lambda_1$ and $\eta$ (based upon the hardness of learning sparse parity with noise) and nearly optimal in terms of $\lambda_2$ for an SQ (Statistical Query) algorithm, in the sense that any SQ algorithm needs at least sub-exponential dependence on $\lambda_2$ (given that the dependence on other parameters is not changed --- e.g. obviously there is a $2^n$ time algorithm to learn this problem). In particular, this shows that our results for learning feedforward networks under $\ell_{\infty}$ are close to tight even in this application, where the input distribution is related to the label.
\begin{theorem}[Informal version of Theorem \ref{thm:lowerboundSQ}]
Let $\mathcal{F}$ be the class of parities on $[n - 1]$. As before, $\lambda_2$ refers to the maximum $\ell_1$-norm into any hidden unit and we choose parameters so that $\lambda_2 = poly(n)$ and $\|w\|_1 = poly(n)$. There exists $\epsilon > 0$ so that no SQ algorithm with tolerance $n^{-\lambda_2^{\epsilon}}$ and access to $n^{\lambda_2^{\epsilon}}$ queries can learn $\mathcal{F}$ with error less than $1/4$.
\end{theorem}
We also show (Theorem~\ref{thm:degeneracy-needed}) that the $\eta$-nondegeneracy condition is required to achieve nontrivial guarantees even if we are only interested in distribution learning (i.e. in TV), assuming the hardness of learning sparse parity with noise.

\section{Supervised RBMs} \label{sec:srbm}
Since in many applications the input data to a classifier is clearly very structured (e.g. images, natural language corpuses, data on networks, etc.), it is interesting to consider the behavior of classification algorithms under structural assumptions on the data.
RBMs are one (relatively simple) generative model which can generate interesting structured data. This suggests the idea of learning ``supervised RBMs'', as proposed by Hinton \citep{hinton2012practical}, where we assume the input and label are drawn from an RBM joint distribution, so that predicting the label is a feedforward network by Theorem~\ref{thm:node-prediction}; in this model the label is just a special visible unit in the RBM. Based on the previous discussion about computational lower bounds, we know that assuming the input to a feedforward network comes from the corresponding RBM does not in general make learning easier, but we know that in RBMs there are very natural assumptions we can make to avoid these computational issues. Our final result is of exactly this flavor, showing how we can learn the supervised RBM under a ferromagneticity-related condition faster than is possible if we did not have a distributional assumption.

In order to emphasize the special role of the node which we want to predict, we will adopt a modified notation where the visible unit which we want to learn to predict is labeled $Y$ and all other visible units are still labeled $X$. More precisely, we model the joint distribution over input features $X$ valued in $\{\pm 1\}^{n_1}$, latent features $H$ valued in $\{ \pm 1\}^{n_2}$  and label $Y \in \{\pm 1\}$ as,
\[ \Pr[X = x,H = h,Y = y] \propto \exp\left(\langle x, Wh \rangle + \langle h, w\rangle y + \langle b^{(1)}, x \rangle + \langle b^{(2)}, h \rangle + b^{(3)}y\right) \]
where the \emph{weight matrix} $W$ is a non-negative $n_1 \times n_2$ matrix, $w$ is an arbitrary $n_1$ dimensional vector and \ $b^{(1)} \in \mathbb{R}^{n_2}$, $b^{(2)} \in \mathbb{R}^{n_2}$ and $b^{(3)} \in \mathbb{R}$ are arbitrary. Given the latent variables $H$, $w$ can be seen as the linear predictor for $Y$.

\begin{theorem}[Informal Version of Theorem \ref{thm:srbm_main}]\label{thm:smart-regression-intro}
Suppose the interaction matrix $W$ is ferromagnetic with minimum edge weight $\alpha$. Further suppose one of the RBMs induced by conditioning on $Y=1$ or $Y=-1$ is a $(\lambda, \lambda)$-RBM. Then there exists an algorithm that learns the predictor $Y$ that minimizes logistic loss up to error $\epsilon$. The algorithm has sample complexity $m = n_1^2\exp(\lambda)^{\exp(O(\lambda))}(1/\alpha)^{O(1)} \log(n_1/\delta)/\epsilon^2$ and has runtime $poly(m)$.
\end{theorem}
Our main algorithm can be broken down into three main steps: (1) Use greedy maximization of conditional covariance $\cov^\avg$ to first learn the two-hop neighborhood $\mathcal{N}(i)$ of each observed variable $i$ w.r.t. the hidden layer conditioned on the label (see Algorithm \ref{algo:rbmferrosup}), (2) For each observed variable $X_i$, learn the conditional law of $X_i \mid X_{\mathcal{N}(i)}, Y$ using regression, and (3) Use the estimated distribution to compute $\E[Y|X]$. Step (1) leverages tools from \citep{bresler2019learning,goel2019learning} but considers a setting where the RBM may in fact have some amount of negative correlation, as $w$ has arbitrary signs and is allowed to have large norm. Step (2) can be achieved by simply looking at the conditional law under the empirical distribution; this is efficient as we learn small neighborhoods. 

In step (3), we can make use of the following useful trick (a version of which can be found in \citep{hinton2012practical}): we already have enough information to derive the law of $Y \mid X$ since we know the marginal law of $Y$ (the fraction of $+$ and $-$ labels) and the law of $X \mid Y$. However, naively carrying out the Bayes law calculation is difficult because it involves partition functions (which are in general NP-hard to approximate, see e.g. \citep{SlySun}). We avoid computing the partition function by observing that if we define $f_1,f_2$ such that $\Pr(X,Y) \propto \exp(f_1(X) \bone(Y = 1) + f_2(X) \bone(Y = -1) + by)$, then
the law of $Y \mid X$ follows a logistic regression model where
\[ \E[Y \mid X] = \tanh\left(\frac{f_1(X) - f_2(X)}{2} + b\right) \]
for some constant $b \in \mathbb{R}$. Therefore if we know $f_1,f_2$ up to additive constants (which we can derive from the Fourier coefficients learned in (2)), we can simply fit a logistic regression model from data to learn $h$ plus the missing constants, and we can prove this works using fundamental tools from generalization theory. We refer the reader to Appendix \ref{sec:feedforward-from-rbm} for additional details.

 \begin{algorithm}
   \caption{$\textsc{LearnSupervisedRBMNbhd}(u, \tau, \mathcal{S})$ (Adapted from \citep{bresler2019learning,goel2019learning})}\label{algo:rbmferrosup}
\begin{algorithmic}[1]
  \STATE Set $S := \phi$
   \STATE Set $i^*= \argmax_v \widehat{\cov}_{\mathcal{S}}^{\avg}(u,v|S, Y)$, and $\eta^* = \max_v \widehat{\cov}_{\mathcal{S}}^{\avg}(u,v|S, Y)$
   \IF {$\eta^* \ge \tau$}
    \STATE $S = S \cup \{i^*\}$
  \ELSE 
  \STATE Go to Step 8
  \ENDIF
  \STATE Go to Step 2
   \STATE For each $v \in S$, if $\widehat{\cov}_{\mathcal{S}}^{\avg}(u, v| S \backslash \{v\}, Y ) < \tau$, remove $v$ ({\em Pruning step})
   \STATE Return $S$
\end{algorithmic}
\end{algorithm}
 Observe that under the given distributional assumptions, our algorithm has runtime complexity polynomial in the input dimension in contrast to Theorem \ref{thm:regression-intro} where the run time scales as $n^{\Omega(\lambda)}$.   A simple example which shows the algorithm from this Theorem will outperform any algorithm without distributional assumptions (like Theorem~\ref{thm:regression-intro}) is given in Remark~\ref{rmk:ferromagnetic-advantage}.


\section{Discussion: Comparison to Prior work on Learning Neural Networks}
In the neural network learning literature, various works prove positive results that either (1) work for any distribution with norm assumptions or (2) require strong distributional assumptions. The result of Theorem~\ref{thm:regression-intro} falls into the category (1) and the result of Theorem~\ref{thm:smart-regression-intro} falls into category (2).

We first discuss the relation of Theorem~\ref{thm:regression-intro} to other previous works of type (1). Perhaps the most closely related works are \citep{shalev2011learning,Zhang,goel2016reliably,goel2018learning}. All of these works assume the input is bounded in $\ell_2$ norm and give learning results based on kernel methods; of course, these results could be applied under the assumption of $\ell_{\infty}$-bounded input, by using the inequality $\|x\|_{2} \le \sqrt{n} \|x\|_{\infty}$ and rescaling the input to have norm $1$. For comparison, the best result in the $\ell_2$ setting with $\tanh$ activation is given in \citep{goel2018learning}, but this result (as is essentially necessary based on the known computational hardness results) has exponential dependence on the $\ell_2$ norm of the weights in the hidden units, so doing such a reduction just using norm comparison bounds gives a runtime sub-exponential in dimension. Therefore it is indeed crucial for us to give a new analysis adapting to learning with input bounded in $\ell_{\infty}$. An interesting feature of this setting (as mentioned above) is that the kernel trick does not seem to be as useful for improving the runtime as the $\ell_2$ setting, where it seems genuinely better than writing out the feature map \citep{goel2016reliably,goel2018learning}. 


Due to the generality of direction (1), it is hard to design efficient algorithms. This further motivates direction (2), however, making the right distributional assumptions which allow for efficient learning while being well-motivated in context of real world data can be very challenging. Most prior work has been limited to the Gaussian input \citep{tian2017analytical,soltanolkotabi2017learning,brutzkus2017globally,zhong2017recovery,li2017convergence,du2018gradient} or symmetric input \citep{goel2018learning,ge2019learning} assumptions which are not satisfied by real world data. The works of \citep{mossel2016deep,malach2018provably} gave results for some simple tree-structured generative models. There has been some work in defining data based notions such as eigenvalue decay \citep{goel2017eigenvalue} and score function computability \citep{gao2019learning} to get efficient results.  Our assumption for Theorem~\ref{thm:smart-regression-intro} in contrast exploits sparsity and nonnegative correlations among the input features conditional on the output label.

\section{Experiments}\label{sec:exp}
\begin{figure}[t!]
    \centering
    \includegraphics{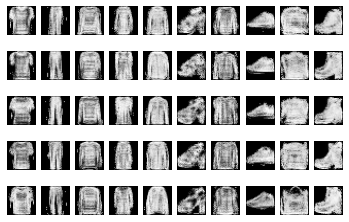}
    \caption{Five i.i.d. samples for each FashionMNIST class, drawn from the trained model by Gibbs sampling.}
    \label{fig:fashionmnist}
\end{figure}
In this section we present some simple experiments on MNIST and FashionMNIST to confirm that our method performs reasonably well in practice. In these experiments, we implemented the supervised RBM learning algorithm from Theorem~\ref{thm:smart-regression-intro} which makes use of the classification labels provided in the training data set. This algorithm outputs both a classifier (which predicts the label given the image) and also a generative model (which can sample images given a label).

For classification, we allowed the logistic regression (described as ``step (3)'' above) to fit not just the bias term but also coefficients on the sum of Fourier coefficients for each pixel (an input of dimension $768 \times 10 = 7680$), since the runtime of the logistic regression step is almost negligible anyway. This is useful because it allows greater dynamic range in the influence of each pixel.

We observed a test accuracy of $97.22 \pm 0.16\%$ on MNIST; the training accuracy was $99.9\%$ and we trained the logistic regression for 30 epochs (same as steps) of L-BFGS with line search enabled. For FashionMNIST, we obtained a test accuracy of $88.84 \pm 0.31\%$; the training accuracy was $92.19\%$ and we trained the logistic regression for 45 epochs with L-BFGS as before. Overall training took a bit less than an hour each on a Kaggle notebook with a P100 GPU.
Both datasets have $60,000$ training points and $10,000$ test; in both experiments we used a maximum neighborhood size of $12$, and stopped adding neighbors if the conditional variance shrunk by less than $1\%$.

For context, we note that our accuracy on MNIST is better than what we would get using standard training methods for RBMs and logistic regression for classification; \cite{gabrie2015training} reports accuracies of approximately $95\%$ for CD and $96\%$ using a more sophisticated TAP-based training method. The results are also around as good or better than what is achieved using many classical machine learning methods on these datasets \cite{xiao2017fashion}; for example, logistic regression achieves error $91.7\%$ and $84.2\%$ and polynomial kernel SVM achieves error $89.7\%$ and $97.6\%$ \cite{xiao2017fashion}. Of course, none of these results are as good as specialized deep convolutional networks (over $99\%$ on MNIST). In contrast to other approaches using linear models such as kernel SVM, our approach also learns a generative model. Being able to sample from the generative model can give some insight into how the model classifies.

To evaluate the performance of the learned RBM as a generative model, we generated samples using Gibbs sampling starting from random initialization and run for 6000 steps. As is common practice, we output the probabilities generated in the last step instead of the sampled binary values, so that the result is a normal greyscale image. We display the resulting samples in Figures~\ref{fig:fashionmnist} and \ref{fig:mnist} (for reference, see randomly sampled training datapoints in Appendix~\ref{apdx:reference-images}): we note that the model successfully generates samples with diversity, as in Figure~\ref{fig:fashionmnist} the model generates handbags both with and without handles, and in Figure~\ref{fig:mnist} it renders both common styles for drawing the number 4.

It is clear that the model fails to generate as detailed of patterns exhibited in real FashionMNIST images since in our training algorithm, we represent a gray pixel as a random combination of black and white, so a checkerboard pattern of black and white and a patch of grey are not well-distinguished. We do this to ensure that our setup is comparable to classic RBM training \cite{hinton2012practical}. It is potentially possible to fix this by adding spins over larger alphabets (e.g. real-valued) to the model.




\subsection*{Acknowledgements}
This work was done in part while the authors were visiting the Simons Institute for the Theory of Computing for the Summer 2019 program on the Foundations of Deep Learning. A substantial part of the work was done while SG was a graduate student at UT Austin.

SG is supported by the JP Morgan AI Phd Fellowship. FK is supported in part by NSF award CCF-1453261 and Ankur Moitra's Packard Foundation Fellowship. AK is supported by NSF awards CCF-1909204 and CCF-1717896. 

\bibliography{bib}
\bibliographystyle{plainnat}
\newpage
\appendix

\section{Outline of the Appendix}
Here we briefly outline the contents of each remaining section; each bold heading in the text below corresponds to a new section.

\paragraph{Appendix \ref{sec:connections}. Connections between Distribution Learning and Prediction in RBMs} In this section we show that if you have learned the distribution of an RBM, then you have also in principle learned how to predict the output of corresponding feedforward networks. These feedforward networks are induced from a ``self-supervised'' prediction task: predicting the spin at node $i$ given observations of all other spins. This connection leverages a classical observation in probabilistic inference: inference in all tree-structured graphical models has an exact solution known as Belief Propagation (see e.g. \cite{pearl2014probabilistic,mezard2009information}); perhaps surprisingly, this observation is useful even though the RBM itself is not tree structured.
Conversely, in the next subsection we give quantitative bounds showing that sufficiently good predictors for this self-supervised objective for every node $i$ allows us to recover the distribution of the corresponding RBM.

\paragraph{Appendix \ref{app:learningffnet}. Guarantees for Learning Feedforward Networks (with arbitrary distribution).}
In this section we prove upper and lower bounds for learning one-layer feedforward networks with $f_{\beta}$ activations in the hidden units and inputs $X$ drawn from an arbitrary distribution such that $\|X\|_{\infty} \le 1$. 

In the first two subsections, we prove the needed approximation-theoretic results about our class of activations $f_{\beta}$, giving approximation results with uniform guarantees over the entire interval $\beta \in [0,1]$. In the special case of $\beta = 0$, $f_{\beta} = \tanh$ and the needed result has essentially already been proved in the work of \cite{shalev2011learning}. As explained in the first subsection, by a classical result of Bernstein (Theorem~\ref{thm:bernstein} below) it turns out that analyzing approximation theory for functions analytic on $[-1,1]$ is equivalent to analyzing the function's extension into the complex plane. We develop the needed complex-analytic estimates (which crucially are uniform in $\beta$) in the following subsection.
We note that the authors of \cite{shalev2011learning} did not use Bernstein's result to prove their bound; their analysis of the $\beta = 0$ case is longer because they more or less reproduce the steps from the proof of the upper bound of Bernstein's Theorem. 

After solving the approximation-theoretic question, we use them in an $\ell_1$-regression based algorithm for learning feedforward networks, using an explicit polynomial feature map and the logistic version of the Lasso with its corresponding nonparametric generalization bounds. We derive the needed $\ell_1$-norm bound in a clean way from the approximation-theoretic results using in part a Lemma of \cite{sherstov2012making}, previously used in \cite{goel2016reliably}.  This proves Theorem~\ref{thm:regression-intro}. In the last subsection, we
prove that this result is nearly optimal under the hardness of sparse parity with noise, even in the case of $\tanh$ networks, using two different ways to construct a parity out of $\tanh$ units: one is a well-known construction from \cite{hajnal1993threshold}, the other is based on Taylor series expansion and is related to the MRF-to-RBM embedding result established in \cite{bresler2019learning}.

\paragraph{Appendix \ref{app:learnRBM}. Learning RBMs by Learning Feedforward Networks.} In this section, we show how to derive structure recovery results (i.e. recovery of Markov blankets) for RBMs by using the feedforward network learning results developed in the previous section. Assuming $\eta$-nondegeneracy, we show how to learn the structure of the network by doing simple regression tests, e.g. comparing the minimal logistic loss achieved predicting node $i$ from all other nodes to the loss when node $j$ is excluded from the input. This proves Theorem~\ref{thm:rbm-structure-recovery-informal}. We explain in more detail in Remark~\ref{rmk:sampling-regimes} how this result is a significant improvement over previous results in interesting regimes where we know that the RBM can actually be sampled from in polynomial time.
Based on this, we prove a result for learning the distribution: by Theorem~\ref{thm:rbm-structure-recovery-informal} this reduces to the case where the structure is known, so by proving a good estimate (Lemma~\ref{lem:finite-regression}) on the convergence of the natural predictor of $X_i$ given its neighbors, the empirical conditional expectation and using the tools developed in Section~\ref{sec:dist-from-predict} gives the result. A key point here is that the empirical conditional expectation converges at a much faster rate than e.g. relying on Theorem~\ref{thm:rbm-regression}, which gives better sample complexity guarantees.

Finally, we again prove some computational hardness results. We establish that the algorithm's dependence is essentially optimal in terms of $\eta$ and $\|w\|_1$ by using the Taylor-series based sparse parity construction from \cite{bresler2019learning}, related to the construction used above for $\tanh$ networks. For the dependence on $\lambda_2$, the hidden unit $\ell_1$-norm, we use a third, different construction of parity from \cite{martens-rbm-representation} for the RBM setting; this construction is not amenable to adding noise, but we are able to prove a lower bound on the runtime in terms of $\lambda_2$ for all SQ (Statistical Query) algorithms (see e.g. \cite{blum1994weakly}).

\paragraph{Appendix \ref{sec:feedforward-from-rbm}. Learning a Feedforward Network by Learning RBMs.}
In this section, we prove Theorem~\ref{thm:smart-regression-intro}, which lets us learn to predict in supervised RBMs under a natural conditional ferromagneticity condition in a provably more computationally efficient way than applying distribution-agnostic methods for learning feedforward networks like Theorem~\ref{thm:regression-intro}. In Remark~\ref{rmk:ferromagnetic-advantage} we give a simple example where the gap is provable and explain the (in this case) simple intuition as to how the approach of Theorem~\ref{thm:smart-regression-intro} uses the structure of the input data in a favorable way.

The idea of this learning algorithm is essentially to use Bayes rule to reduce computing the posterior on the label (i.e. $\Pr(Y | X)$) to computing the conditional likelihood of the observed $X$ under the two possible values of the label. In some situations where the conditional law of $Y | X$ is very simple, this approach may be overkill as it requires to model the law of $X$; however, we are interested in the setting where the label $Y$ may have a large, complicated effect on $X$ so this approach seems perfectly reasonable.
An obvious issue with using Bayes rule in this way is that even if the the RBM is already known perfectly, computing the normalizing constant for the conditional distribution under $Y = +$ or $Y = -$ in such a model is $\#$BIS-Hard \cite{goldberg2007complexity}. Fortunately, for our application we show that we can estimate the needed ratio of normalizing constants from the data using a simple variant of logistic regression.

What remains is to learn how to estimate the conditional log-likelihoods i.e. $\Pr(X | Y)$. Fortunately, even though under our assumptions the original RBM was not ferromagnetic, the conditional models we get by applying Bayes rule are indeed ferromagnetic so we can apply the methods developed in \cite{goel2019learning} for learning such a model. Here we need the results of \cite{goel2019learning} and not the earlier work of \cite{bresler2019learning} as we expect the external fields in the resulting model to be inconsistent (have differing signs depending on the site). Once the structure is recovered, we can learn the coefficients of the log-likelihood using the results established in the previous section based on fast convergence of the empirical condition expectation, and using these coefficients we can accurately estimate $\Pr(X | Y)$ for the application of Bayes rule. 

\paragraph{Appendix \ref{apdx:reference-images}. Additional Experimental Data.} In this section we include reference images from both datasets along with samples generated by our algorithm trained on MNIST.

\section{Connections between Distribution Learning and Prediction in RBMs}\label{sec:connections}

To our knowledge, Theorem~\ref{thm:node-prediction} has not been previously noted in the literature on RBMs. However, this is not the first time connections between RBMs and message passing algorithms for inference has been investigated: for example, the work of \cite{welling2003approximate} extensively studied the use of message passing algorithms (i.e. Belief Propagation and related algorithms) for estimating the mean and covariance matrix of nodes in an RBM, and the work of \cite{gabrie2015training} used the related TAP approximation to derive better alternatives to constrastive divergence for training RBMs in practice. The key conceptual difference is that in these works, their goal is to solve a much harder problem (e.g. estimating marginals and $\log Z$) which is well-known to be NP-hard in general. In contrast, for our application to learning the relevant task ends up being predicting one node from the others, which it turns out is \emph{not}  computationally difficult if we know the model --- conditioning on the other nodes breaks all cycles in the graph, which is the obstacle that makes inference difficult in general.


\subsection{Conditional Law Derivation}\label{apdx:conditional-law}
In this Appendix we give, for the reader's convenience, a self-contained derivation of the conditional law \eqref{eqn:node-prediction} described in Theorem~\ref{thm:node-prediction} for $\E[X_i | X_{\sim i}]$ from \eqref{eqn:cond-joint-law}. As described in the proof of the Theorem, the result is obtained as a special case of the Belief Propagation algorithm as described in a number of references, including \cite{mezard2009information,pearl2014probabilistic}, which is derived by performing a more general version of this calculation. First recall that the joint conditional law on $X_i,H$ condiditioned on $X_{\sim i}$ is given by \eqref{eqn:cond-joint-law}:
\[ \Pr(X_i = x_i, H = h | X_{\sim i} = x_{\sim i}) \propto \exp\left(x_i(b^{(1)}_i + \sum_j W_{ij} h_j) + \langle W_{\sim i}^t x_{\sim i} + b^{(2)}, h \rangle\right). \]
The computation proceeds by rewriting this measure with respect to a ``cavity'' measure where all terms involving $X_i$ are removed.
For each hidden unit $j$, define a corresponding probability measure
\[ \mu_{H_j \to X_i}(h_j) \propto \exp\left(\sum_{k \ne i} W_{kj} x_j h_j + b^{(2)}_j h_j \right)\]
under which $\sum_j h_j \mu_{H_j \to X_i}(h_j) = \tanh(\sum_k W_{kj} x_j + b^{(2)}_j)$
and rewrite the joint probability over $X,H$ as
\[ \Pr(X_i = x, H = h |X_{\sim i} = x_{\sim i}) \propto \exp\left(x_i(b^{(1)}_i + \sum_j W_{ij} h_j)\right) \prod_j \mu_{H_j \to X_i}(h_j). \]
Now we compute that
\begin{align*} 
&\Pr[X_i = x_i | X_{\sim i} = x_{\sim i}]\\
&= \sum_h x_i \Pr(X_i = x_i, H = h | X_{\sim i} = x_{\sim i}) \\
&\propto \sum_h \exp\left(x_i(b^{(1)}_i + \sum_j W_{ij} h_j)\right) \mu_{H \to X_i}(h) \\
&= \exp(x_i b^{(1)}_i) \prod_{j = 1}^{n_2} (\cosh(W_{ij}) + \sinh(x_i W_{ij})\tanh(\sum_{k \ne i} W_{kj} x_j + b^{(2)}_j)) \\
&\propto  \exp(x_i b^{(1)}_i) \prod_{j = 1}^{n_2} (1 + x_i \tanh(W_{ij})\tanh(\sum_{k \ne i} W_{kj} x_j + b^{(2)}_j)) \\
&= \exp\left(x_i b^{(1)}_i + \sum_{j = 1}^{n_2} \log (1 + x_i \tanh(W_{ij})\tanh(\sum_{k \ne i} W_{kj} x_j + b^{(2)}_j))\right)
\end{align*}
where we used $\propto$ to ignore constants of proportionality independent of $x_i$ and in the third line we used Lemma~\ref{lem:mgf-calc} below. Therefore if we use that \[ \log(1 + \beta x_i) = \frac{1}{2} \log \frac{1 + \beta x_i}{1 - \beta x_i} + \frac{1}{2} (\log(1 + \beta x_i ) + \log(1 - \beta x_i)) = \tanh^{-1}(\beta x_i) + \frac{1}{2} (\log(1 + \beta) + \log(1 - \beta)) \]
where we see the last term does not depend on $x$, we can compute that
\begin{align*}
    &\E[X_i = x_i | X_{\sim i} = x_{\sim i}] \\
    &= \frac{\sum_{x_i} x_i \exp\left(x_i b^{(1)}_i + \sum_{j = 1}^{n_2} \log (1 + x_i \tanh(W_{ij})\tanh(\sum_{k \ne i} W_{kj} x_j + b^{(2)}_j))\right)}{\sum_{x_i} \exp\left(x_i b^{(1)}_i + \sum_{j = 1}^{n_2} \log (1 + x_i \tanh(W_{ij})\tanh(\sum_{k \ne i} W_{kj} x_j + b^{(2)}_j))\right)} \\
    &= \frac{\sum_{x_i} x_i \exp\left(x_i b^{(1)}_i + \sum_{j = 1}^{n_2} x_i \tanh^{-1}(\tanh(W_{ij})\tanh(\sum_{k \ne i} W_{kj} x_j + b^{(2)}_j))\right)}{\sum_{x_i} \exp\left(x_i b^{(1)}_i + \sum_{j = 1}^{n_2} x_i \tanh^{-1}(\tanh(W_{ij})\tanh(\sum_{k \ne i} W_{kj} x_j + b^{(2)}_j))\right)} \\
    &= \tanh\left(b^{(1)}_i + \sum_{j = 1}^{n_2} \tanh^{-1}(\tanh(W_{ij})\tanh(\sum_{k \ne i} W_{kj} x_j + b^{(2)}_j))\right)
\end{align*}
where in the final step we used that $\tanh(z) = \frac{e^z - e^{-z}}{e^z + e^{-z}}$. From this we get \eqref{eqn:node-prediction} by plugging in the definition of $f_{\beta_{ij}}$.
\begin{lemma}\label{lem:mgf-calc}
For any $z \in \mathbb{R}$ we have the formula for moment generating function of a recentered Bernoulli:
\begin{align*} 
\E_{X \sim Ber_{\pm}(\tanh(z))}[\exp(\lambda X)] = \cosh(\lambda) + \sinh(\lambda)\tanh(z)
\end{align*}
where $Ber_{\pm}(\mu)$ denotes the distribution of a $\{\pm 1\}$-valued random variable with mean $\mu$.
\end{lemma}
\begin{proof}
First recall that
$\E_{X \sim Rad}[\exp(\lambda X)] = \cosh(\lambda)$ and $\E_{X \sim Rad}[X\exp(\lambda X)] = \tanh(\lambda)$.
Therefore
\begin{align*}
\E_{X \sim Ber_{\pm}(\tanh(z))}[\exp(\lambda X)] 
&=\E_{X \sim Rad}\left[e^{\lambda X} \frac{e^{z X}}{\cosh(z)}\right] \\
&= \frac{\cosh(z + \lambda)}{\cosh(z)} \\
&= \frac{\cosh(z)\cosh(\lambda) + \sinh(z)\sinh(\lambda)}{\cosh(z)} \\
&= \cosh(\lambda) + \sinh(\lambda)\tanh(z).
\end{align*}
\end{proof}
\subsection{2-layer Tanh Neural Network as Bayes-Optimal Prediction in an RBM}\label{subsec:node-prediction}
In particular, \eqref{eqn:node-prediction} lets us realize \emph{any} standard 2-layer $\tanh$ neural network as the Bayes-optimal predictor in an RBM in a natural limit where the number of hidden neurons goes to infinity, but the effect of each hidden neuron is very small, so that the $\ell_1$ norm of the weights going into the top neuron stays bounded by a constant. Each hidden unit in the neural network corresponds in a direct way to several duplicated hidden units in the RBM. The construction is given explicitly in the next Lemma; we will not use the statement explicitly but use it to develop intuition for \eqref{eqn:node-prediction}.
\begin{lemma} \label{lem:express}
Suppose that $g(x) = \tanh\left(u_0 + \sum_{j = 1}^T u_j \tanh\left(M_{j0} + \sum_k M_{jk} x_k \right)\right)$ where $x$ is $n$-dimensional, i.e. $g$ is a 2-layer neural network with $\tanh$ activations. Then
\[ g(x) = \lim_{K \to \infty} \tanh\left(u_0 + \sum_{i=1}^K\sum_{j = 1}^{T} \tanh(u_j/K) f_{|u_j/K|}\left(M_{j0} + \sum_k M_{jk} x_k\right) \right), \]
so by \eqref{eqn:node-prediction} from Theorem~\ref{thm:node-prediction} the restriction of $f$ to $\{\pm 1\}^n$ is the Bayes-optimal predictor of a visible unit in an RBM with $n + 1$ total visible units where the activations of the other visible units are known. 
\end{lemma}
\begin{proof}
This follows from the observation in Remark~\ref{rmk:f-activation} and from Theorem~\ref{thm:node-prediction} by building the corresponding RBM with $KT$ hidden units.
\end{proof}
\subsection{Distribution learning bounds from prediction bounds}\label{sec:dist-from-predict}
In this section, we show how good estimates of the conditional prediction functions can be used in a direct way to recover the joint distribution of the RBM in total variation distance.
\begin{algorithm}
   \caption{$\textsc{DistributionFromPredictors}$}
\begin{algorithmic}[1]
    \STATE For every $i$ we suppose we are given $\hat{f}_i : \{\pm 1\}^n \to \mathbb{R}$ and set $\widehat{\mathcal{N}}(i)$ such that $\hat{f}_i$ is a predictor of node $i$ from other nodes that depends only on those in the set $\widehat{\mathcal{N}}(i)$
    \STATE Define $\mathcal{S} := \{S : \exists i, S \subset \widehat{\mathcal{N}}(i) \}$
    \FOR{$S \in \mathcal{S}$}
        \STATE For all $i \in S$, define $\hat{w}_{S,i} := \E_{X \sim Uni(\{\pm 1\}^n)}[\tanh^{-1}(\hat f_i(X)) X_{S \setminus i}]$.
        \STATE Define $\hat{w}_S := \frac{1}{|S|} \sum_{i \in S} \hat{w}_{S,i}$.
    \ENDFOR
    \STATE Return the MRF with unnormalized pmf $\exp\left(\sum_{S \in \mathcal{S}} \hat{w}_S X_S\right)$.
\end{algorithmic}
\end{algorithm}
\begin{lemma}[\cite{santhanamW}]\label{lem:skl-formula}
Suppose $P,Q$ are distributions over random variable $X$ valued in $\{\pm 1\}^n$.
If $P(x) \propto \exp(\sum_S p_S X_S)$ and $Q(x) \propto \exp(\sum_S q_S X_S)$ then
\[ \SKL(P,Q) = \sum_S (p_S - q_S)(\E_P[X_S] - \E_Q[X_S]). \]
where $\SKL(P,Q) = \KL(P,Q) + \KL(Q,P)$ is the symmetrized KL divergence.
\end{lemma}
\begin{proof}
From the definition we see
\[ \SKL(P,Q) = \E_P\left[\log\frac{P(x)}{Q(x)}\right] - \E_Q\left[\log \frac{P(x)}{Q(x)}\right] = \E_P\left[\sum_S (p_S - q_S) X_S\right] - \E_Q\left[\sum_S (p_S - q_S) X_S\right] \]
so using linearity of expectation proves the result.
\end{proof}
The following definition captures the level of contiguity $P$ has with the uniform measure when looking at small sets of coordinates.
\begin{defn}
For any distribution $P$ on $\{\pm 1\}^n$ and $d \le n$ we define
\[ \delta_P(d) := \inf_{|S| \le d} \inf_{x_S} 2^{|S|} P(X_S = x_S). \]
\end{defn}
\begin{lemma}
For any function $f$ which depends on at most $d$ coordinates,
\[ \E_P[f(X)^2] \ge \delta_P(d) \E_{X \sim \{\pm 1\}^n}[f(X)^2] \]
\end{lemma}
The following Lemma is a standard observation used in most previous works on learning Ising models including \citep{Bresler,Vuffray,KlivansM} and others.
\begin{lemma}\label{lem:spin-freedom}
A $(\lambda_1,\lambda_2)$-bounded RBM satisfies $\delta_P(d) \ge (1 - \tanh(\lambda_1))^d$.
\end{lemma}
\begin{proof}
In the $d = 1$ case this follows from the law of total expectation as $\E[X_i | H,X_{\sim i}] = \tanh(b^{(1)}_i + \sum_j W_{ij} H_j)$ and the term inside the $\tanh$ has magnitude at most $\lambda_1$ by definition. For general $d$ the result follows by induction, by using the above argument for a single spin and then applying the induction hypothesis to the model where than spin is plus and where that spin is minus, since these models are also $(\lambda_1,\lambda_2)$-bounded RBMs. 
\end{proof}
\begin{lemma}\label{lem:dist-from-predictors}
Let $\hat{P}$ denote the distribution returned by Algorithm~\textsc{DistributionFromPredictors} and let $P$ 
be the true distribution. Let $\log P(x) = \sum_S w_S x_S$ and $\log \hat{P}(x) = \sum_S \hat{w}_S x_S$ be the Fourier expansions of the log-likelihoods.
 Then
\begin{align*} \SKL(\hat{P},P) &\le \sum_S |w_S - \hat{w}_S| \\
&\le \sum_i \frac{2^{|\mathcal{N}(i)|/2 + 1}}{\sqrt{\delta_P(|\mathcal{N}(i) \cup \widehat{\mathcal{N}}(i)|)}} \sqrt{\E_{X'}[(\tanh^{-1}(\hat f_i(X')) - \tanh^{-1}(\E_P[X_i | X_{\sim i}]))^2]} \end{align*}
\end{lemma}
where $X' \sim Uni(\{\pm 1\}^n)$.
\begin{proof}
Define $w_S$ to be the true coefficient in the true MRF potential. By Lemma~\ref{lem:skl-formula} and Holder's inequality we know $\SKL(P,\hat{P}) \le 2 \sum_S |\hat{w}_S - w_S|$. Then by Jensen's inequality and the Cauchy-Schwarz inequality,
\begin{align*} 
\sum_S |\hat{w}_S - w_S|
&\le \sum_S \frac{1}{|S|} \sum_{i \in S} |\hat{w}_{S,i} - w_S|\\
&= \sum_{i} \sum_{S: i \in S} \frac{1}{|S|} |\hat{w}_{S,i} - w_S|\\
&\le \sum_{i} 2^{|\mathcal{N}(i)|/2} \sqrt{\sum_{S : i \in S} (\hat{w}_{S,i} - w_S)^2}.
\end{align*}
Now using Plancherel's theorem \citep{odonnell2014}, the fact that $f_i(x) = \tanh\left(\sum_{S : i \in S} w_S x_{S \setminus \{i\}}\right)$, and the definition of $\delta_P(d)$ gives the result. 
\end{proof}

\section{Guarantees for Learning Feedforward Networks (with Arbitrary Distribution)} \label{app:learningffnet}
In this section we prove upper and lower bounds for learning one-layer feedforward networks with $f_{\beta}$ activations in the hidden units and inputs $X$ drawn from an arbitrary distribution such that $\|X\|_{\infty} \le 1$. 


\subsection{Preliminaries: Optimal Approximation of Analytic Functions}
Identify $\mathbb{C}$ with $\mathbb{R}^2$ by taking $x$ to be real and $y$ to be the imaginary component of a complex number $z$. Define $\mathcal E_{\rho}$ to be
  the region bounded by the ellipse in $\mathbb{C} = \mathbb{R}^2$ centered at the origin with equation
  $\frac{x^2}{a^2} + \frac{y^2}{b^2} = 1$ with semi-axes
  $a = \frac{1}{2}(\rho + \rho^{-1})$ and
  $b = \frac{1}{2} |\rho - \rho^{-1}|$; the focii of the ellipse are
  $\pm 1$. In the present context, this is sometimes referred to as a \emph{Bernstein ellipse}.  For an arbitrary function $f : [-1,1] \to \mathbb{R}$, let $E_D(f)$ denote the error of the best polynomial
  approximation of degree $D$ in infinity norm on the interval $[-1,1]$ of $f$, i.e.
  \begin{equation}\label{eqn:ed-err}
  E_D(f) := \min_{P : \deg(P) \le D} \max_{x \in [-1,1]} |f(x) - P(x)|. 
  \end{equation} The following theorem of Bernstein exactly characterizes the asymptotic rate at which $E_D(f)$ shrinks:
  \begin{theorem}[Theorem 7.8.1, \cite{constructive-approximation}]\label{thm:bernstein}
    Let $f$ be a function defined on $[-1,1]$. Let $\rho_0$ be the supremum of all $\rho$ such that
    $f$ has an analytic extension on the interior of $\mathcal E_{\rho}$. Then
\[ \limsup_{D \to \infty} \sqrt[D]{E_D(f)} = \frac{1}{\rho_0} \]  
    where we interpret the rhs as $\infty$ when $\rho_0 = 0$.
  \end{theorem}
  For the definition of what it means for the function to be analytic on a region of the complex plane, we refer to a text on complex analysis such as \cite{stein2010complex}.
  For our application we need only the upper bound and we need a quantitative estimate for finite degree $d$.
  In the proof of the upper bound in \cite{constructive-approximation}, the following result is proved:
\begin{theorem}[Quantitative Variant of Theorem 7.8.1, \cite{constructive-approximation}]\label{thm:bernstein-quantitative}
  Suppose $f$ is analytic on the interior of $\mathcal E_{\rho_1}$ and $|f(z)| \le M$ on the closure of $\mathcal E_{\rho_1}$.  Then
\[ E_D(f) \le \frac{2M}{\rho_1 - 1} \rho_1^{-D}. \]
\end{theorem}
\noindent
This quantitative variant was previously used in \citep{koehler2018comparative} as part of a construction of low-degree approximations to the ReLU activation with specific properties.
Note that when applying this theorem, we should center $f$ so that the constant $M$ is small, since adding constants to $f$ will obviously not change $E_d(f)$.
\subsection{Approximation Guarantees for $f_{\beta}$ Family of Activations}
Recall that the activations $f_{\beta}$ were defined in Theorem~\ref{thm:node-prediction} to be $f_{\beta}(x) = \frac{1}{\beta} \tanh^{-1}(\beta \tanh(x))$. Recall that if $\beta = 1$ then
$f_{\beta}(x) = x$ so the function is analytic everywhere on $\mathbb{C}$, and if $\beta = 0$ is is $\tanh$ so it is meromorphic. For the remaining values of $\beta \in (0,1)$, the function $f_{\beta}$ is slightly more complicated (it has branch cuts), however we show it is still nicely behaved near the real line.
\begin{lemma}\label{lem:f-analytic}
For $\beta \in [0,1]$ the function $f_{\beta}$ is analytic on the strip $\{ x + iy : |y| < \pi/2 \}$.
\end{lemma}
\begin{proof}
Observe that
\[ f'_{\beta}(z) = \frac{1 - \tanh^2(z)}{1 - \beta^2 \tanh^2(z)}.. \]
Since $\tanh$ is analytic except at points of the form $z = \frac{\pi}{2} i + \pi k i$, the only other possible poles are solutions to $\beta^2 \tanh^2(z) = 1$, i.e. solutions to $\tanh(z) = \pm 1/\beta$. Recalling that $\tanh^{-1}(z) = \frac{1}{2} (\log(1 + z) - \log(1 -z ))$ and taking into account the branch cut from $(-\infty,0]$ for the logarithm, we see that the solutions to $\tanh(z) = 1/\beta$ are of the form
\[ z = \frac{1}{2} \log\frac{1 + 1/\beta}{1/\beta - 1} + \frac{\pi i}{2} + k \pi i \]
and for $\tanh(z) = -1/\beta$ of the form
\[ z = \frac{1}{2} \log \frac{1/\beta - 1}{1 + 1/\beta}  + \frac{\pi i}{2} + k \pi i \]
for $k \in \mathbb{Z}$. In particular we see that $f'_{\beta}$ is analytic on the strip $\{x + iy : |y| < \pi/2\}$ so $f_{\beta}$ is as well (since the region is simply connected, this can be proved by path integration \citep{stein2010complex}).
\end{proof}
To get a quantitative upper bound we will need to bound (the centered version of) $f_{\beta}$ on the Bernstein ellipse, which will require us to back away from the singularities of $f'_{\beta}$ on the lines $y = \pm \pi/2$. The following Lemma proves that $f'_{\beta}$ is uniformly bounded in a slightly smaller region:
\begin{lemma}\label{lem:fprime-bound}
For all $\beta \in [0,1]$, $|f'_{\beta}(z)| \le 2$ everywhere on the closed strip $\{x + iy : |y| \le \pi/4 \}$.
\end{lemma}
\begin{proof}
Observe that
\begin{align*} f'_{\beta}(z) 
= \frac{1 - \tanh^2(z)}{1 - \beta^2 \tanh^2(z)}
&= \frac{\cosh^2(z) - \sinh^2(z)}{\cosh^2(z) - \beta^2 \sinh^2(z)} \\ 
&= \frac{1}{1 + (1 - \beta^2) \sinh^2(z)} 
= \frac{1}{1 + (1 - \beta^2) \frac{\cosh(2z) - 1}{2}}
\end{align*}
using the identies $\cosh^2(x) - \sinh^2(x) = 1$ and $\sinh^2(z) = \frac{\cosh(2z) - 1}{2}$. Since $\cosh(2x + 2iy) = \frac{e^{2x + 2iy} + e^{-2x - 2iy}}{2}$ we see that under the assumption $|y| \le \pi/4$ that $\cosh(2x + 2iy)$ lies in the right half plane, therefore $|1 + (1 - \beta^2) \frac{\cosh(2z) - 1}{2}| \ge |1 - (1 - \beta^2)/2| \ge 1/2$ which proves the result.
\end{proof}
\begin{lemma}\label{lem:f-apx-guarantee}
For any $\beta \in [0,1]$, arbitrary $h \in \mathbb{R}$, and any $R \ge 0$,
\[ E_D(f_{\beta}(Rx + h)) \le \frac{4R(1 + 2R)}{(1 + 1/2R)^{D}} \]
\end{lemma}
\begin{proof}
Just for this proof define $g_{\beta,h}(x) := f_{\beta}(Rx + h) - f_{\beta}(h)$.
We prove this bound by application of Bernstein's theorem. By Lemma~\ref{lem:f-analytic} we know that $f_{\beta}$ is analytic on the strip $\{x + iy : |y| < \pi/2 \}$ so in 
particular it is analytic on the closed strip $\{x + iy : |y| \le \pi/4\}$, and by Lemma~\ref{lem:fprime-bound} we know that $|f'_{\beta}| \le 2$ on the closed strip. 

We now compute $\rho$ so that $R \mathcal E_{\rho}$ is contained in the latter strip. We solve
\[ \frac{1}{2}(\rho - \rho^{-1}) = \frac{\pi}{4R} \]
which gives $\rho^2 - \frac{\pi}{2 R} \rho - 1 = 0$ so $\rho = \frac{\pi/2R + \sqrt{\pi^2/4R^2 + 4}}{2} > 1 + 1/2R$. Since $|g_{\beta,h}'(z)| \le R|f'_{\beta}| \le 2R$ on the closure of the ellipse, it follows by the mean-value theorem that $|g_{\beta,h}| \le 2(1 + 1/2R)R \le 1 + 2R$ on $\mathcal E_{1 + 1/2R}$ and applying Theorem~\ref{thm:bernstein-quantitative} gives the result.
\end{proof}
\subsection{Learning Feedforward Networks under $\ell_{\infty}$ Bounded Input}
Since the final activation in our network is $\tanh$, we recall some useful facts about logistic regression and the logistic loss which we will use.
\begin{defn}
The \emph{logistic loss} is defined to be
\[ \ell(v,y) := \log(1 + e^{-2 v y}). \]
We note that the factor of 2 in the exponent and the normalization differ depending on convention.
\end{defn}
The following facts about the logistic loss which can be checked from the definition (or see a reference such as \citep{shalev2014understanding}):
\begin{fact}\label{fact:logistic-loss}
The following are true if $y \in \{\pm 1\}$ is fixed:
\begin{enumerate}
\item $\ell(v,y)$ is convex and 2-Lipschitz in $v$.
\item $\ell(v,y) = -\log \Pr(\hat{Y} = y)$ where $\hat{Y}$ is a $\{\pm 1\}$-valued random variable with expectation $\tanh(v)$.
\item $\frac{\partial}{\partial v} \ell(v,y) = \frac{-2y e^{-2v y}}{1 + e^{-2vy}}$ and $\frac{\partial^2}{\partial v^2} \ell(v,y) = \frac{2}{1 + \cosh(2v)}$. 
\end{enumerate}
Furthermore if $Y$ is a $\{\pm 1\}$-valued random variable (and $v$ is deterministic) then
\begin{enumerate} \setcounter{enumi}{3}
\item 
$\E_Y \ell(v,Y) = \KL(\mathcal{L}(Y),\mathcal{L}(\hat{Y})) + H(Y)$
where $\hat{Y}$ is defined above, $\mathcal{L}(Y)$ denotes the law of random variable $Y$, $\KL$ denotes the Kullback-Liebler divergence and $H$ denotes the Shannon entropy.
\end{enumerate}
\end{fact}
We recall the following Theorem which states the agnostic learning guarantee for fitting $\ell_1$-constrained predictors in logistic loss, i.e. the logistic version of the Lasso:
\begin{theorem}[Theorem 26.15 of \cite{shalev2014understanding}]\label{thm:l1-regression}
Suppose that $X$ is a random vector in $\mathbb{R}^n$ such that $\|X\|_{\infty} \le 1$ almost surely and $Y$ is an arbitrary $\{\pm 1\}$-valued random variable. Then with probability at least $1 - \delta$, simultaneously for all $w$ with $\|w\|_1 \le R$ it holds that
\[ \hat{\E}[\ell(w \cdot X, Y)] \le \E[\ell(w \cdot X, Y)] + 4R\sqrt{\frac{2 \log(2n)}{m}} + 2R \sqrt{\frac{2\log(2/\delta)}{m}} \]
where $\hat{\E}$ denotes the empirical expectation over $m$ i.i.d. copies $(X_1,Y_1),\ldots,(X_m,Y_m)$ of $(X,Y)$.
\end{theorem}
In order to bound the $\ell_1$ norm of our predictor we will need the following Lemmas:
\begin{lemma}[\cite{sherstov2012making}, Lemma 2.13 of \cite{goel2016reliably}]\label{lem:sherstov}
Suppose $p(x) = \sum_{i = 0}^D \beta_i x$ and $|p(x)| \le M$ for $x \in [-1,1]$, then
$\sum_{i = 0}^D \beta_i^2 \le (D + 1)(4e)^{2D} M^2$.
\end{lemma}
\begin{lemma}\label{lem:l1-coeff-bound}
Suppose that $p(x) = \sum_{i = 0}^D a_i (w \cdot x)^i = \sum_{\alpha} u_{\alpha} x^{\alpha}$. Then
\[ \sum_{\alpha} |u_{\alpha}| \le \sqrt{\sum_i a_i^2} (1 + \|w\|_1)^{D}. \]
\end{lemma}
\begin{proof}
For any multi-index $\alpha$ let $w_{\alpha} := \prod_{i \in \alpha} w_i$ and observe by the multinomial theorem 
\[ p(w \cdot x) = \sum_i a_i (w \cdot x)^i = \sum_i a_i \sum_{|\alpha| = i} {i \choose \alpha} w_{\alpha} x^{\alpha}. \]
Therefore by the triangle inequality, multinomial theorem, and Cauchy-Schwarz inequality
\[ \sum_{\alpha} |u_{\alpha}| \le \sum_i |a_i| \sum_{|\alpha| = i} {i \choose \alpha} |w_{\alpha}| = \sum_i |a_i| \|w\|_1^i \le \sqrt{\sum_i a_i^2 \sum_i \|w\|_1^{2i}} \le \sqrt{\sum_i a_i^2} (1 + \|w\|_1)^{d} \]
where in the last step we used $1 + x^2 + x^4 + \cdots + x^k \le (1 + x)^k$ for $x \ge 0$.
\end{proof}
\begin{theorem}\label{thm:rbm-regression}
Suppose that $Y$ is a random variable valued in $\{\pm 1\}$, $X$ is a random vector such that $\|X\|_{\infty} \le 1$ almost surely and
\[ \E[Y | X] = \tanh\left(b^{(1)} + \sum_j w_j f_{\beta_j}\Big(b^{(2)}_j + \sum_{k} W_{jk} X_k\Big)\right) \]
where $b^{(1)} \in \mathbb{R}$, $\beta_j \in [0,1]$, $w$ is an arbitrary real vector and $W$ is an arbitrary real matrix. Let $W_j$ denote \emph{column} $j$ of $W$. Then $\ell_1$-constrained regression on the degree $D$ monomial feature map $\varphi_D(x) \mapsto \left(\prod_{i \in S} X_i\right)_{|S| \le d}$ with $\ell_1$ constraint
\[ \|w\|_1 \le R := |b^{(1)}| + \sqrt{D + 1}(4e)^{D + 1} \sum_j |w_j| (1 + \|W_j\|_1)^{D + 1} \]
returns a predictor $\hat{w}$ such that with probability at least $1 - \delta$,
\begin{align*} &\E[\ell(\hat{w} \cdot \varphi_d(X), Y)] - \E[\ell(v^*(X), Y)] \\
&\le 8\sum_j |w_j| \frac{\|W_j\|_1 + 2\|W_j\|_1^2}{\left(1 + 2/\|W_j\|_1\right)^D} +  4R\sqrt{\frac{2D\log(2n)}{m}} + 2R \sqrt{\frac{2\log(2/\delta)}{m}} \end{align*}
where $v^*(X) := \tanh^{-1}(\E[Y | X]) = b^{(1)} + \sum_j w_j f_{\beta_j}\Big(b^{(2)}_j + \sum_{k} W_{jk} X_k\Big)$ is the minimizer of the expected logistic loss over all measurable functions of $X$.
The runtime is $poly(n^D)$.
\end{theorem}
\begin{proof}
The fact that $v^*(X)$ is the minimizer of the logistic loss $\E[\ell(h(X),Y)]$ over all $X$-measurable functions $h$ can be seen from Fact~\ref{fact:logistic-loss}. To derive the bound we combine the approximation-theoretic guarantees developed in the previous section with the $\ell_1$ guarantee for logistic Lasso.

For the approximation step, define $w^*$ so that $w^* \cdot \varphi_d(X)$ is given by replacing each activation $f_{\beta_{j}}$ by its best polynomial approximation $P_j$ on the interval $[b_j^{(2)} - \|W_j\|_1, b_j^{(2)} + \|W_j\|_1]$. By the triangle inequality and Lemma~\ref{lem:f-apx-guarantee}, for any $x \in \{\pm 1\}^n$,
\[ |v^*(x) - w^* \cdot \varphi_d(x)| \le \sum_j |w_j||(f_{\beta_j} - P_j)(b_j^{(2)} + \sum_k W_{jk} x_k)| \le 4\sum_j \frac{|w_j| \left(\|W_j\|_1 + 2\|W_j\|_1^2\right)}{\left(1 + 2/\|W_j\|_1\right)^D}.  \]
Since the logistic loss is $2$-Lipschitz (Fact~\ref{fact:logistic-loss}.1), this implies that
\begin{equation}\label{eqn:apx-error}
\E[\ell(w^* \cdot \varphi_D(X), Y)] \le \E[\ell(v^*(X),Y)] + 8\sum_j \frac{|w_j| \left(\|W_j\|_1 + 2\|W_j\|_1^2\right)}{\left(1 + 2/\|W_j\|_1\right)^D}.
\end{equation}
Combining Lemma~\ref{lem:fprime-bound}, Lemma~\ref{lem:sherstov} and Lemma~\ref{lem:l1-coeff-bound} and using the triangle inequality shows that $\|w^*\|_1 \le R$ where $R$ is as specified in the Theorem statement. Then applying Theorem~\ref{thm:l1-regression} and combining it with \eqref{eqn:apx-error} gives the desired inequality bounding the error of the predictor $\hat{w}$.
\end{proof}
To simplify usage of this Theorem, we give the following slightly less precise bound which will be used from now on:
\begin{corollary}\label{corr:rbm-regression-simple}
In the same setting as Theorem~\ref{thm:rbm-regression}, if we assume that $\|W_j\|_1 \le \lambda$ for every $j$ and $\lambda \ge 2$, then with probability at least $1 - \delta$, $\E[\ell(\hat{w} \cdot \varphi_d(X), Y)] - \E[\ell(v^*(X), Y)] \le \epsilon$ as long as the number of samples $m$ satisfies $m = \Omega((|b^{(1)}|^2 \lambda^{O(D)})\log(2n/\delta))$ where $D = O(\lambda \log(\|w\|_1\lambda/\epsilon))$ and the runtime of the algorithm is $poly(n^D)$.
\end{corollary}
\begin{proof}
In order to make the first term of the bound on $\E[\ell(\hat{w} \cdot \varphi_d(X), Y)] - \E[\ell(v^*(X), Y)]$ at most $\epsilon/2$, we can upper bound it by $O(\|w\|_1 \lambda^2/(1 + 2/\lambda)^D)$ and see that it suffices to take $D = \Omega(\lambda \log(\|w\|_1\lambda/\epsilon))$. Then $R = |b^{(1)}| + \exp(O(D)) \|w\|_1 \lambda^{D + 1} = |b^{(1)}| + \lambda^{O(D)}$ so it suffices to take $m = \Omega((|b^{(1)}|^2 + \lambda^{O(D)})\log(2n/\delta))$ 
\end{proof}
\begin{remark}
In the analysis of Theorem~\ref{thm:rbm-regression} we did not concern ourselves with the exact constants in the runtime. However, if we are interested in optimizing the runtime it should be noted that instead of getting a precise estimate of the empirical risk minimizer when computing the logistic regression, one can achieve a similar statistical guarantee by using a single pass of stochastic mirror descent/exponentiated gradient (see reference text  \cite{bubeck2015convex}), e.g. as used in \cite{KlivansM} where the needed high-probability guarantees can be found.
\end{remark}
\subsection{Nearly Matching computational lower bounds}
In this section, we show that the runtime guarantee of Corollary~\ref{corr:rbm-regression-simple} is close to optima: more precisely its runtime is optimal in $\|w\|_1$ and $\epsilon$ up to a $\log \log$ factor in the exponent, and also that at least sub-exponential dependence on $\lambda$ is required. 
We first recall the definition of this problem and a standard hardness assumption for learning sparse parity with noise. We phrase it in terms of a testing problem versus the uniform distribution, which is equivalent to a learning formulation (i.e. recovering $S$ below), by boosting the probability of success and using a standard reduction of removing one coordinate at a time and testing (see e.g. \cite{Valiant}).
\begin{defn}
  The \emph{$k$-sparse parity with noise} distribution is the following
  distribution on $(X,Y)$ parameterized by $\eta \in (0,1/2)$
  and an unknown subset $S$ of size $k$:
  \begin{enumerate}
  \item Sample $X \sim \mathrm{Unif}(\{- 1,+1\}^n)$.
  \item With probability $1/2 + \eta$, set $Y = \prod_{s \in S} X_s$,
    and with probability $1/2 - \eta$, set $Y = (-1) \prod_{s \in S} X_s$.
  \end{enumerate}
  The \emph{$k$-sparse parity with noise} problem is to test between
  the uniform and $k$-sparse parity with noise with sum of probability of Type I and Type II errors upper bounded by $0.01$, given access to an oracle
  which generates samples from one of the two distributions.
\end{defn}
\begin{assumption}[Hardness of learning sparse parity with noise]\label{assumption:spwn}
Suppose $k_n$ is an arbitrary sequence of positive integers with $k_n = o(n^{1 - \epsilon})$ for any $\epsilon > 0$ and $n$ growing, any algorithm which solve the $k$-sparse parity with noise testing problem must have runtime $n^{\Omega(k_n)}$.
\end{assumption}
The reason for the condition $k_n = o(n^{1 - \epsilon})$ is simply because the number of sets of size $n$ is $2^n$, not $n^n$, so small correction factors in the exponent are needed when $k$ is comparable to $n$. The best known algorithm for learning sparse parity with noise runs in time $n^{0.8 k_n}$ \cite{Valiant}. 
\begin{theorem}\label{thm:net-hardness}
In the setting of Corollary~\ref{corr:rbm-regression-simple} and under Assumption~\ref{assumption:spwn}, for $\lambda \le 2$ there exists families of models (one with $\epsilon$ a constant, one with $\|w\|_1$ a constant) where a runtime of
\[ n^{\Omega\left(\frac{\log(\|w\|_1/\epsilon)}{\log\log(\|w\|_1/\epsilon)}\right)} \]
is needed for any algorithm to achieve $\epsilon$ error with high probability, regardless of its sample complexity and even in the case of $\tanh$ activations ($\beta_j = 0$ for all $j$). There also exists a sequence of models with $\lambda = \Theta(n \log(n))$ and $\|w\|_1 = O(\sqrt{n})$ which requires runtime
\[ n^{\Omega(\sqrt{\lambda/\log^3(\lambda)} \log \|w\|_1)} \]
to achieve error $\epsilon = 0.01$ with high probability.
\end{theorem}
\begin{proof}
We first show a lower bound of $n^{\Omega(\log(\|w\|_1/\epsilon))}$ for a family of models where $\lambda \le 1$. Recall we are proving a lower bound in the $\beta_j = 0$ case where all activations are $\tanh$.
The lower bound is shown by building a parity function out of $\tanh$ functions exactly using a simple taylor series expansion argument, under the assumption that the input to the network is in the hypercube $\{\pm 1\}^n$. The construction proceeds in a similar fashion to the sparse parity with noise lower bound for learning RBMs of bounded hidden degree established in \citep{bresler2019learning}.
We first describe the construction of a parity function on boolean inputs $x_1,\ldots,x_k$. It suffices to build this parity with a small (constant-size) coefficient, since we can repeat it to make the coefficient larger.
We start from the fact that
\[ \tanh(z) = 2 \sum_k \frac{(-1)^k}{\pi^{2k + 2}} (1 - 1/4^{k + 1}) \zeta(2k + 2) z^{2k + 1} \]
for $|z| < \pi/2$ and recall that the Riemann $\zeta$ function does not vanish on even integers \citep{stein2010complex}, so every coefficient in this expansion is nonzero. Furthermore it is known that $\zeta(n) \to 1$ as $n \to \infty$, since this follows from the power series definition of $\zeta(s) = \sum \frac{1}{n^s}$, so we can write
\[ \tanh(z) = \sum_k a_{2k + 1} z^{2k + 1} \]
where $a_{2k + 1} \ne 0$ for any $k$ and $|a_{2k + 1}| = \Theta(1/\pi^{2k + 2})$.
From this we can see that for some constant $c \ne 0$,
\[ x_1 \cdots x_{2k + 1} = c \frac{(2k + 1)^{2k + 1}}{a_{2k + 1}} \tanh\left(\frac{x_1 + \cdots + x_{2k + 1}}{2k + 1}\right) + p(x) \]
where $p(x)$ is of degree at most $k - 1$, using that $x_i^2 = 1$ for all $i$ on the hypercube; here the constant $c$ (which is close to $1$) is a fixed correction factor to handle the small effect of maximum-degree terms coming from expanding higher order terms in the $\tanh$ power series. We can inductively rewrite each of the highest-order coefficients of $p$ in terms of $\tanh$ and lower order monomials: this ultimately gives us a way to write parity as a linear combination of $\tanh$ functions. Using this, we can rewrite $\tanh(\frac{1}{4} x_1 \cdots x_{2k + 1})$ as a two-layer $\tanh$ network with $\|w\|_1 = k^{O(k)}$ and $\lambda \le 1$. Taking $\epsilon = 1/16$ and using the hardness of $k$-sparse parity with noise, we get that the runtime for learning the corresponding network is at least $n^{\Omega(k)} = n^{\Omega(\log(\|w\|_1)/\log\log(\|w\|_1))}$. 

We can similarly prove a lower bound of $n^{\Omega(\log(1/\epsilon)/\log\log(1/\epsilon))}$ for constant $\lambda,\|w\|_1$ by using the same method to convert $\tanh(\eta x_1 \cdots x_{2k + 1})$ into a two-layer network and by taking $\eta = k^{-\Theta(k)}$ so that the $\ell_1$ norm of the coefficients is shrunk to be at most $1$. Taking $\epsilon = \Theta(\eta)$ and using the sparse parity with noise lower bound as above gives the result.

Finally, we give a lower bound showing exponential dependence on $\lambda$ is necessary. We use the well-known fact that a parity can be written as a small sum of threshold functions \citep{hajnal1993threshold}. For $k$ even,
\[ x_1 \cdots x_{k} = \bone[x_1 + \cdots + x_k \ge -k] - 2(\bone[x_1 + \cdots + x_k \ge -k + 1] - \bone[x_1 + \cdots + x_k \ge -k + 2] + \cdots) \]
with a total of $2k - 1$ terms in the sum on the rhs. We now consider replacing each threshold function with the approximation $\bone[a \ge b] \approx \frac{1 + \tanh(\lambda'(a - b + 1/2))}{2}$ for some $\lambda' > 0$. Note that the error of this approximation for a singe threshold unit and integers $a,b$ is maximized when $a - b = 0$ where the error is $\frac{1 - \tanh(\lambda'/2)}{2} = O(e^{-\lambda'})$. Therefore by Holder's inequality, the error in approximating $x_1 \cdots x_k$ by replacing all of the threshold functions is $O(k e^{-\lambda'})  = O(k e^{-\lambda/(k + 1/2)})$, where we used that $\lambda = (k + 1/2)\lambda'$ where $\lambda$ is the hidden node $\ell_1$ norm as used previously. By adding a $\tanh$ nonlinearity on top of the approximate parity, this gives an approximate construction of sparse parity with noise.

Taking $k = \sqrt{n}$ and $\lambda = \Theta(k^2 \log(n))$ we see that the resulting model is $\TV$-distance $n^{-\Theta(k)}$ from sparse parity with noise, so any algorithm with runtime $c n^{-\Theta(k)}$ cannot distinguish this model from sparse parity with noise with probability better than 75\% for sufficiently small constant $c > 0$. From the assumed hardness of learning sparse parity with noise, any algorithm succeeding to distinguish this model from the uniform distribution with sufficiently small error probability requires runtime $n^{\Omega(k)} = n^{\sqrt{\lambda/\log^3(\lambda)} \log \|w\|_1}$.
\end{proof}
\section{Learning RBMs by Learning Feedforward Networks} \label{app:learnRBM}
\subsection{Structure and Distribution Learning Guarantees}
In this section we discuss application of the prediction guarantees from the previous section to structure and distribution learning. As motivation, recall that in undirected graphical models the \emph{Markov blanket} or \emph{neighborhood} of a node $i$, the minimal set of nodes which separate node $i$ from the rest of the model in the underlying graph, is one of the most interesting pieces of information to learn about a node. By the Markov property, node $i$ interacts directly only with nodes in its Markov blanket, in the sense that $X_i$ is conditionally independent of all other nodes $X_k$ given the values of nodes $X_j$ for all $j$ in the markov blanket of $i$. Learning the markov blanket of all nodes, equivalently learning the underlying graph of the Markov Random Field, is referred to as \emph{structure learning}. It is also known (see e.g. \cite{bresler2019learning}) that once we have performed structure learning, distribution learning (e.g. in total variation distance) becomes a conceptually straightforward task as it can typically be reduced to solving low-dimensional regression problems.

As explained in the introduction, learning the structure requires a non-degeneracy condition on neighbors (recall the definition of $\eta$-nondegeneracy from above). In the introduction, we stated that if all edges are $\eta$-nondegenerate then we can learn the structure perfectly; in the next Theorem, we state a slightly more precise result giving the result we can successfully test between non-neighbors and $\eta$-nondegenerate neighbors, without requiring nondegeneracy on the entire model. Since our guarantee holds with high probability, using the union bound it immediately gives a result for structure recovery under $\eta$-nondegeneracy.

\begin{theorem}\label{thm:rbm-structure-recovery}
Let $i$ and $j$ be two visible nodes in a $(\lambda_1,\lambda_2)$-bounded RBM. Let $H_0$ be the hypothesis that nodes $i$ and $j$ are not two-hop neighbors and $H_1$ the hypothesis that nodes $i$ and $j$ are $\eta$-nondegenerate two-hop neighbors. Given $\delta > 0$ and $m = \Omega(\lambda_2^{O(D)} \log(2n/\delta))$ i.i.d. samples where $D = O(\lambda_2 \log(\lambda_1 \lambda_2/\eta))$, we can test in time $poly(n^D)$ between $H_0$ and $H_1$ with sum of Type I and Type II errors upper bounded by $\delta$.
\end{theorem}
\begin{proof}
We run the following testing procedure:
\begin{enumerate}
    \item Run the $\ell_1$ regression algorithm from Theorem~\ref{thm:node-prediction} to predict $X_i$ 
    from $X_{\sim i}$ and from $X_{\sim i,j}$.
    \item Repeat the previous step with $i$ and $j$ reversed.
    \item If the decrease in prediction accuracy for removing $i$ or $j$ is at least $3\eta/4$ in either step 1 or step 2, reject $H_0$.
\end{enumerate}
That this works follows by combining Theorem~\ref{thm:node-prediction} and Corollary~\ref{corr:rbm-regression-simple}, by choosing $\epsilon = \eta/8$ 
under $H_0$ the difference in prediction error is at most $2\epsilon$ whereas under $H_1$ it must be at least $\eta - 2\epsilon$.
\end{proof}
Assuming that all 2-hop neighbors in the RBM are $\eta$-nondegenerate, the above Theorem lets us recover the structure of the RBM (its 2-hop neighborhoods) in time $poly(n^D)$. In the following remark, we explain how large $D$ is in the regimes where we know polynomial time sampling from the RBM is possible:
\begin{remark}[Comparison to polynomial time sampling regimes]\label{rmk:sampling-regimes}
Dobrushin's uniqueness criterion is probably the most well-known sufficient condition for sampling to be possible in polynomial time in a general pairwise model. Dobrushin's condition is that for every node $i$, the total $\ell_1$-norm of the edges touching node $i$ is at most $1$, where the mixing time guarantees for Glauber dynamics become worse as the maximum norm approaches $1$ (see \cite{levin2017markov}). This condition is tight in the example of the Ising model on the complete graph (Curie-Weiss), or for the bipartite complete graph (i.e. dense RBM) with all  edge weights positive and equal and an equal number of visible and hidden units.

Under Dobrushin's uniqueness criterion on the RBM, we have that $\lambda_1,\lambda_2 \le 1$ so $D = O(\log(1/\eta))$. As mentioned above, we cannot compute $\eta$ in terms of just the edge weights for general models, but if we for example assume the model is $d$-regular and has all edge weights equal to $+1/d$ and no external field then it is not too hard to show that $\eta = \Omega(1/d^2)$ (see e.g. \cite{bresler2019learning}), so in this case the overall runtime is $n^{\log(d)}$. We expect that under Dobrushin's condition $\eta = \Omega(1/d^2)$ except in perhaps some rare degenerate situations.
This means the runtime is improved by an exponential factor in the exponent compared to what one gets by just applying the RBM to MRF reduction, since learning $d$-wise MRFs is known to require $n^d$ time in general \citep{KlivansM}. 

In some other interesting contexts, it is also known that polynomial time sampling can only be guaranteed when $\lambda_1,\lambda_2 = O(1)$: for antiferromagnetic Ising models on bounded degree graphs with equal edge weights the sharp result is known for every $d$ \citep{sinclair2014approximation,galanis2016inapproximability,SlySun} and embedding these Ising models as RBMs with hidden nodes of degree 2 in a straightforward way gives models with $\lambda_1,\lambda_2 = O(1)$ and $\eta = \Omega(1/d^2)$ (see Example~\ref{example:ising-eta} above).
\end{remark}

For distribution learning we will need the following technical Lemma, which is proved in Appendix~\ref{apdx:finite-regression} using the local Rademacher complexity framework \citep{bartlett2005local}. Informally it says that if $X$ is a random variable with a density with respect to the uniform measure on $\{\pm 1\}^n$ that is lower bounded by a constant, then given a number of samples $m$ which is large with respect to the size of the domain the natural estimator of $\tanh^{-1}(\E[Y | X])$ has error which converges at a $1/m$ rate, which generalizes the case of estimating the (exponential-family parameterization of) mean, the $n = 0$ case, in a natural way. Since the bound depends exponentially on $n$, we will only apply it in settings where we expect $n$ is small. Similar bounds are used in previous works including \citep{BreslerMosselSly,Bresler} and proved using different methods, though they are not quite as optimized (e.g. deriving this result from Lemma 3.2 of \cite{Bresler} would give a $1/\gamma^2$ dependence); this bound can be shown to be optimal up to constants.
\begin{lemma}\label{lem:finite-regression}
Suppose that $X$ is a random variable valued in $\{\pm 1\}^n$ with $\Pr(X = x) \ge \gamma/2^n$ for every $x$ and $Y$ is a random variable valued in $\{\pm 1\}$. Suppose that $|\E[Y | X]| \le r$ for $r < 1$. Let $\hat{\E}[Y | X]$ be the empirical conditional expectation of $Y$ given $X$ based upon $m$ i.i.d. samples of $(X,Y)$ and define $h(X) := \min(\max(\E[Y | X],r),-r)$. Then with probability at least $1 - \delta$, 
\[ \E[(\tanh^{-1}(h(X)) - \tanh^{-1}(\E[Y | X]))^2] \lesssim \frac{2^n/\gamma + \log(1/\delta)}{(1 - r^2)^2 m} \]
where $\lesssim$ denotes inequality up to an absolute constant.
\end{lemma}
We present the proof of this lemma in the subsequent subsection. From this Lemma we straightforwardly get the right result for learning a sparse RBM with known 2-hop neighborhoods.

\begin{algorithm}
   \caption{$\textsc{DistributionFromStructure}$}
\begin{algorithmic}[1]
    \STATE We assume for every node $i$ we are given a recovered neighborhood $\widehat{\mathcal{N}}(i)$. $\widehat{\mathcal{N}}(i)$
    \STATE For every node $i$ with neighborhood $\widehat{\mathcal{N}}(i)$, let $f_i(X) := \widehat{\E}[X_i | X_{\widehat{\mathcal{N}}(i)}]$ be the empirical conditional expectation of $X_i$ given $X_{\widehat{\mathcal{N}}(i)}$.
    \STATE Return the output of Algorithm~\textsc{DistributionFromPredictors} run with these $f_i$.
\end{algorithmic}
\end{algorithm}
\begin{lemma}\label{lemma:dist-recovery-known}
For any $(\lambda_1,\lambda_2)$-bounded RBM where the maximum two-hop degree of any visible node is at most $d_2$ and where $\|b^{(1)}\|_{\infty} \le B$, for $\delta > 0$ and $m = \Omega\left(n^2 \left(\frac{2}{(1 - \tanh(\lambda_1))}\right)^{d_2 + 1}\log(n/\delta)/\epsilon^4\right)$ we have that with probability at least $1 - \delta$, Algorithm~\textsc{DistributionFromStructure} given $m$ samples and $\mathcal{\widehat{N}}(i) = \mathcal{N}(i)$ for every $i$ returns a distribution $\hat{P}$ which is $\epsilon$-TV close to the distribution of the RBM. Furthermore, if $w_S,\hat{w}_S$ are as defined as in Lemma~\ref{lem:dist-from-predictors} then
\[ 2\TV(P,\hat{P})^2 \le \SKL(P,\hat{P}) \le \sum_S |w_S - \hat{w}_S| \le \epsilon^2. \]
\end{lemma}
\begin{proof}
By Lemma~\ref{lem:dist-from-predictors}, Lemma~\ref{lem:spin-freedom} and Lemma~\ref{lem:finite-regression} we have
\begin{align*}
\SKL(\hat{P},P) 
&\le \sum_S |w_S - \hat{w}_S| \\
&\le \sum_i \frac{2^{d_2/2 + 1}}{(1 - \tanh(\lambda_1))^{d_2/2}} \sqrt{\E_{X \sim Uni(\{\pm 1\}^n)}[(\tanh^{-1}(h_i(X)) - \tanh^{-1}(\E_P[X_i | X_{\sim i}])^2]} \\
&\le \sum_i \frac{2^{d_2/2 + 1}}{(1 - \tanh(\lambda_1))^{d_2}} \sqrt{\E_{X_{\mathcal{N}(i)}}[(\tanh^{-1}(h_i(X)) - \tanh^{-1}(\E_P[X_i | X_{\sim i}])^2]} \\
&\le \sum_i \frac{2^{d_2/2 + 1}}{(1 - \tanh(\lambda_1))^{d_2}} \sqrt{\frac{2^{d_2}/(1 - \tanh(\lambda_1))^{d_2} + \log(n/\delta)}{(1 - \tanh(\lambda_1)^2)^2 m}}
\end{align*}
and by Pinsker's inequality $\TV(\hat{P},P)^2 \le \SKL(\hat{P},P)/2$ so the result follows.
\end{proof}
\begin{theorem}\label{thm:dist-recovery-rbm}
Suppose that all visible nodes in an RBM which are neighbors in the Markov blanket sense are $\eta$-nondegenerate neighbors, and that maximum 2-hop degree of any visible node is at most $d_2$. Then given $\delta > 0$ and $m = \Omega(\lambda_2^{O(D)} \log(2n/\delta) + n^2 \left(\frac{2}{(1 - \tanh(\lambda_1))}\right)^{d_2 + 1}\log(n/\delta)/\epsilon^4)$ i.i.d. samples where $D = O(\lambda_2 \log(\lambda_1 \lambda_2/\eta))$ samples, Algorithm~\textsc{DistributionFromStructure} run with the set of $\eta$-nondegenerate neighbors output by Theorem~\ref{thm:rbm-structure-recovery} returns with probability at least $1 - \delta$ a distribution which is $\epsilon$-TV close to the true distribution of the RBM.
\end{theorem}
\begin{proof}
This follows by combining Theorem~\ref{thm:rbm-structure-recovery} and Lemma~\ref{lemma:dist-recovery-known}. 
\end{proof}
\begin{remark}
If we do not assume that all neighbors are $\eta$-nondegenerate, then by Theorem~\ref{thm:degeneracy-needed} it is impossible to get a nontrivial distribution learning guarantee assuming the hardness of learning sparse parity with noise, in the sense that the naive approach of forgetting the RBM structure entirely and using MRF learning results (e.g. \cite{KlivansM}) cannot be improved.
\end{remark}
\subsection{Proof of Lemma~\ref{lem:finite-regression}}\label{apdx:finite-regression}
We recall the statement of Lemma~\ref{lem:finite-regression}.
Suppose that $X$ is a random variable valued in $\{\pm 1\}^n$ with $\Pr(X = x) \ge \gamma/2^n$ for every $x$ and $Y$ is a random variable valued in $\{\pm 1\}$. Suppose that $|\E[Y | X]| \le r$ for $r < 1$. Let $\hat{\E}[Y | X]$ be the empirical conditional expectation of $Y$ given $X$ based upon $m$ i.i.d. samples of $(X,Y)$ and define $h(X) := \min(\max(\E[Y | X],r),-r)$. Then with probability at least $1 - \delta$, 
\[ \E[(\tanh^{-1}(h(X)) - \tanh^{-1}(\E[Y | X]))^2] \lesssim \frac{2^n}{\gamma (1 - r^2)^2 m} + \frac{\log(1/\delta)}{(1 - r^2)^2m} \]
We will prove the result by proving the analogous result without the $\tanh^{-1}$ first, as Lemma~\ref{lem:finite-regression-2}. 
The following general result reduces this to computing the local Rademacher complexity of the corresponding function class.
\begin{theorem}[Corollary 5.3 of \cite{bartlett2005local}]\label{thm:local-rademacher}
Suppose that $\mathcal{F}$ is a class of functions from $\mathcal{X}$ to $[-1,1]$ and $\ell(\hat{y},y)$ is a loss which satisfies:
\begin{enumerate}
    \item $\ell$ is $L$-Lipschitz in $\hat{y}$.
    \item There is a constant $B \ge 1$ such that for any random variable $X$ supported on $\mathcal{X}$ and random variable $Y$ on $[-1,1]$
    \[ \E (f(X) - f^*(X))^2 \le B \E[\ell(f(X),Y) - \ell(f^*(X),Y)]\]
    where $f^*(X)$ is a minimizer of $\E[\ell(f(X),Y)]$ which we assume exists.
\end{enumerate}
Then if $\psi(r)$ is a sub-root function (meaning a monotonically increasing non-negative function with $\psi(r)/\sqrt{r}$ monotonically decreasing) such that 
\begin{equation}\label{eqn:subroot}
\psi(r) \ge BL \E \sup_{f \in \mathcal{F}, L^2\E[(f - f^*)^2] \le r} \frac{1}{m} \sum_{i = 1}^m \sigma_i (f - f^*)(X_i) 
\end{equation}
where the $\sigma_i$ are i.i.d. Rademacher random variables, then for any $r \ge \psi(r)$ with probability at least $1 - \delta$
\[ \E[\ell(\hat{f}(X),Y) - \ell(f^*(X),Y)] \lesssim \frac{r}{B} + \frac{(L + B)\log(1/\delta)}{m} \]
where the notation $\lesssim$ hides an absolute constant. 
\end{theorem}
\begin{lemma}\label{lem:finite-regression-2}
Under the same setup as Lemma~\ref{lem:finite-regression},
\[ \E[(h(X) - \E[Y | X])^2] \lesssim \frac{2^n}{\gamma m} + \frac{\log(1/\delta)}{m}. \]
\end{lemma}
\begin{proof}
We consider $\mathcal{F}$ the class of arbitrary functions from $\mathcal{X}$ to $[-r,r]$ and take $\ell(\hat{y},y) := (\hat{y} - y)^2$ to be the square loss so $L = 2$ and $B = 1$ satisfy the conditions above. It is clear from the definition of $h$ that it is the empirical risk minimizer for this function class and loss. Since this class is convex we can take $\psi(r)$ to be defined by the rhs of \eqref{eqn:subroot} (Lemma 3.4 of \cite{bartlett2005local}) and it remains to compute the fixed point of $\psi$. Thus
if we write $g := f - f^*$
\begin{align*} 
\psi(r) 
&= 2 \E \sup_{f : 4\E[g^2] \le r} \frac{1}{m} \sum_{i = 1}^m \sigma_i g(X_i)
\end{align*}
and we observe by the assumption $\Pr(X = x) \ge \gamma/2^n$ that 
\[ \E_X[g^2] \ge \gamma \E_{X' \sim Uni(\{\pm 1\}^n)}[g(X')^2] = \gamma \sum_S \widehat{g}(S)^2 \]
by Plancherel's Theorem \citep{odonnell2014} where $\widehat{g}(S)$ denotes the Fourier coefficient of $g$ corresponding to set $S$, so that $g(x) = \sum_S \widehat{g}(S) x_S$ where $x_S = \prod_{s \in S} x_s$. Therefore by the above, the Cauchy-Schwarz inequality, and Jensen's inequality we have
\begin{align*}
\psi(r)
&= 2 \E \sup_{g : 4\E[g^2] \le r} \frac{1}{m} \sum_{i = 1}^m \sigma_i g(X_i) \\
&\le 2 \E \sup_{g : \sum_S \hat{g}(S)^2 \le r/4\gamma} \frac{1}{m} \sum_S \hat{g}(S) \frac{1}{m} \sum_{i = 1}^m \sigma_i (X_i)_S \\
&\le \sqrt{r/\gamma} \E \frac{1}{m} \sqrt{\sum_S \left(\sum_{i = 1}^m \sigma_i (X_i)_S\right)^2} \\
&\le \frac{\sqrt{r}}{m \sqrt{\gamma}} \sqrt{\E \sum_S \left(\sum_{i = 1}^m \sigma_i (X_i)_S\right)^2} = \frac{\sqrt{r}}{\sqrt{m \gamma}} 2^{n/2}.
\end{align*}
Solving for the fixed point of $r = \frac{\sqrt{r}}{\sqrt{m \gamma}} 2^{n/2}$ gives $r^* = \frac{2^n}{\gamma m}$ so the result follows from Theorem~\ref{thm:local-rademacher}.
\end{proof}
\begin{proof}[Proof of Lemma~\ref{lem:finite-regression}]
Recall that the derivative of $\tanh^{-1}$ at $x$ is $\frac{1}{1 - x^2}$. Therefore on the domain $[-r,r]$ the function $\tanh^{-1}$ is $\frac{1}{1 - r^2}$ Lipschitz. Therefore by the mean value theorem,
\[ \E[(\tanh^{-1}(h(X)) - \tanh^{-1}(\E[Y | X]))^2] \le \frac{1}{(1 - r^2)^2} \E[(h(X) - \E[Y | X])^2] \]
and applying Lemma~\ref{lem:finite-regression-2} gives the result.
\end{proof}
\subsection{Matching Computational Lower Bounds}\label{sec:computational-lower-bounds}
In the following sequence of theorems we show that our runtime guarantees for structure learning of RBMs cannot be significantly improved.
The first result relies in part on the representation of sparse parity with noise given in \cite{bresler2019learning}; this embedding is constructed in a similar way to the first embedding used in Theorem~\ref{thm:net-hardness}. It shows the dependence on $\lambda_1$ and $\eta$ is correct when asking for structure recovery.

\begin{theorem}\label{thm:rbm-hardness-1}
In the same setup as Theorem~\ref{thm:rbm-structure-recovery} and under Assumption~\ref{assumption:spwn}, there exists a family of instances parameterized by $n$ going to infinity with $\lambda_2 \le 2$ such that any algorithm which is able to achieve structure recovery for a model with all neighbors being $\eta$-nondegenerate requires runtime $n^{\Omega(\log(\lambda_1/\eta)/\log\log(\lambda_1/\eta))}$, regardless of its sample complexity.
\end{theorem}
\begin{proof}
In \cite{bresler2019learning}, it was shown that for any fixed constant $\eta$ (say $\eta = 1/8$), there exists an embedding of $k$-sparse parity with noise into an RBM where every hidden unit has incoming edges of total $\ell_1$ norm upper bounded by $2$ (i.e. satisfying $\lambda_1 \le 2$) and there are $2^{O(k)}$ hidden units; it can be checked straightforwardly that for $\eta = 1/8$ that $\lambda_2 = k^{O(k)}$.
Therefore if we fix $\epsilon = \eta/2$ then when assuming the hardness of $k$-sparse parity with noise there is a $n^{\Omega(k)}$ runtime lower bound which matches since $\lambda_2 = e^{O(k)}$. 

For the tightness in $\epsilon$, by making the parity bias $\eta$ exponentially small in $k \log(k)$, it's easy to check that by repeating the construction in \cite{bresler2019learning} that we can make $\lambda_2$ a constant; then to find the parity with noise one needs $\epsilon$ exponentially small in $k \log k$ as well, and the hardness assumption implies the runtime must be $n^{\Omega(k)}$.  
\end{proof}

By tensorizing this construction, we show that the $\eta$-nondegeneracy assumption is required, even if we only care about distribution learning. More precisely, we need it to learn in TV distance with runtime better than the pessimistic $n^{O(d_h)}$ result which follows from viewing the RBM as an unstructured MRF and using the result of \cite{KlivansM}.
\begin{theorem}\label{thm:degeneracy-needed}
There exists a family of RBMs with $n$ nodes, maximum hidden node degree $d_H$, and $\lambda_1,\lambda_2 = O(1)$ such that any algorithm which can learn this family of RBMs within total variation distance at most $1/4$ requires $n^{\Omega(d_H)}$ time.
\end{theorem}
\begin{proof}
The construction in Theorem~\ref{thm:rbm-hardness-1} shows that there exists a family of RBMs given by embedding sparse parity with noise with the desired property, except that the total variation distance is only guaranteed to be $2^{-O(d_H \log(d_H))}$. By building a larger RBM consisting of $2^{d_H \log(d_H)}$ disjoint copies of the original RBM (note that the resulting increase in $n$ is a multiplicative factor independent of the original $n$), we can boost the total variation distance to be arbitrarily close to $1$.
\end{proof}

In order to give lower bounds with respect to $\lambda_2$ for fixed $\eta$, we need a significantly more involved argument. 
We first recall an approximate construction of parity (with low levels of noise) from \cite{martens-rbm-representation}:
\begin{theorem}[Theorem 7 of \cite{martens-rbm-representation}]\label{thm:martens}
There exists an RBM network with $n^2 + 1$ hidden units and weights $poly(n,\log(1/\epsilon))$ such that the marginal distribution $P$ on the visible units satisfies $P(x) \propto e^{f(x)}$ for some $f$ satisfying
\[ \sup_{x \in \{\pm 1\}^n} \left|f(x)/C - x_1 \cdots x_n\right| \le \epsilon \]
where $C > 0$ satisfies $C = poly(\log(n),\log(1/\epsilon))$. 
\end{theorem}
This construction is for a dense parity, but obviously we can make the parity as sparse as we want by adding additional visible units not connected to anything else. More significantly, since the above theorem only constructs an $\epsilon$-approximate instance of parity with noise $\eta = O(1/2 - 1/poly(n,1/\epsilon))$, when $n$ or $1/\epsilon$ is large it does not seem that the resulting distribution is computationally hard to distinguish from the uniform distribution, since Gaussian elimination over $\mathbb{F}_2$ has some chance of succeeding to find the parity. Since we need $\epsilon$ to be small for the model to be indistinguishable from sparse parity with noise, this appears to be a barrier to deriving a hardness result from the above Theorem. 
Instead, we will prove that our result cannot be significantly improved for SQ (Statistical Query) algorithms (for a reference, see \cite{blum1994weakly}).
In the Statistical Query model algorithms do not have access to data, but instead have access to an SQ oracle:
\begin{defn}
An oracle for the statistical query model over distribution $\mathcal{D}$ over $X,Y$  takes input $(g,\tau)$ where $g$ is a function $g : \{\pm 1\}^n \times \{\pm 1\} \to [-1,1]$ and $\tau$ is a tolerance, and gives output $v$ with
\[ |\E_{X,Y \sim D}[g(X,Y)] - v| \le \tau. \]
\end{defn}
Standard arguments, i.e. implementing the needed regressions using standard gradient-based methods for convex optimization shows that our algorithm for learning RBMs can be implemented in the statistical query model (in this case, the separation of $X$ and $Y$ in the definition above is somewhat artificial but we will take $Y$ to be a particular visible unit in the RBM). We will show that statistical query algorithms cannot do better than subexponential dependence on $\lambda_2$. 

The following theorem statements a lower bound for learning concepts of large SQ-dimension in the Statistical Query model. The definition of SQ-dimension can be found in \cite{blum1994weakly}, but for our purposes the only needed fact is that the class of $k$-parities over the uniform distribution $\{\pm 1\}^n$ has SQ-dimension ${n \choose k}$ \citep{blum1994weakly}. 
\begin{theorem}[\cite{blum1994weakly}]\label{thm:sqdim}
Let $\mathcal{F}$ be a class of functions over $\{\pm 1\}^n$ and $D$ a distribution such that $\text{SQ-DIM}(\mathcal{F},D) \ge d \ge 16$. Then if all queries are made with tolerance at least $1/d^{1/3}$, then at least $d^{1/3}/2$ queries are required to learn $\mathcal{F}$ with error less than $1/2 - 1/d^3$ in the statistical query model.
\end{theorem}
\begin{theorem} \label{thm:lowerboundSQ}
Let $S$ be an unknown subset of $[n]$ of size $k$ and containing $n$ and $\mathcal{D}$ is the distribution of the RBM produced by Theorem~\ref{thm:martens} on $S$ where the other $n - |S|$ visible units are isolated and without external field. Let $\mathcal{F}$ be the class of parities on $[n - 1]$. As before, $\lambda_2$ refers to the maximum $\ell_1$-norm into any hidden unit and we choose parameters so that $\lambda_2 = poly(n)$ and $\|w\|_1 = poly(n)$.
There exists $\epsilon > 0$ so that no SQ algorithm with tolerance $n^{-\lambda_2^{\epsilon}}$ and access to $n^{\lambda_2^{\epsilon}}$ queries can learn $\mathcal{F}$ with error less than $1/4$.
\end{theorem}
\begin{proof}
In Theorem~\ref{thm:martens} we take $\epsilon = \exp(-n)$ which gives $\lambda_2 = poly(n)$. The resulting RBM is then within TV distance $\exp(-n)$ of the distribution of a parity over the uniform distribution with a small amount of label noise, so an SQ algorithm for the RBM setting implies an SQ algorithm for learning parity, and the result follows from the lower bound of Theorem~\ref{thm:sqdim}.
\end{proof}
\section{Learning a Feedforward Network by Learning RBMs}\label{sec:feedforward-from-rbm}
In this section we reverse the connection between RBMs and Feedforward networks by using RBMs with certain structural assumptions as a useful \emph{distributional assumption} for learning feedforward network. More formally, we assume our data is generated by the following Supervised RBM.
\begin{defn} \label{def:suprbm}
A \emph{Supervised Restricted Boltzmann Machine} is any joint distribution over random variables $X$ valued in $\{\pm 1\}^{n_1}$, $H$ valued in $\{ \pm 1\}^{n_2}$  and label $Y \in \{\pm 1\}$ of the form
\[ \Pr[X = x,H = h,Y = y] \propto \exp\left(\langle x, Wh \rangle + \langle h, w\rangle y + \langle b^{(1)}, x \rangle + \langle b^{(2)}, h \rangle + b^{(3)}y\right) \]
where the \emph{weight matrix} $W$ is an arbitrary $n_V \times n_H$ matrix and \emph{external fields}/\emph{biases} $b^{(1)} \in \mathbb{R}^{n_1}$, $b^{(2)} \in \mathbb{R}^{n_2}$ and $b^{(3)}$ are arbitrary, and 
$X$ is referred to as the vector of \emph{visible unit} activations and $H$ the vector of \emph{hidden unit} activations.
\end{defn}
We make the following additional assumptions on the parameters of the model.
\begin{assumption}[Minimum Ferromagnetic Interaction] \label{ass:srbm1} For all $i \in [n_1] ,j \in [n_2]$ either $W_{ij} = 0$ or $W_{ij} \ge \alpha$.
\end{assumption}
We do not make any assumption on the weight $w$ to the label. Therefore the model overall is not ferromagnetic.
\begin{assumption}[Sparsity]\label{ass:srbm2} For all $i \in [n_1]$, $\sum_{j = 1}^{n_2} W_{ij} + |b^{(1)}_i| \leq \lambda$ and for either $y = -1$ or $y = 1$, for all $j \in [n_2]$ $\sum_{i = 1}^{n_1} W_{ij} + |b^{(2)}_j + yw_j| \leq \lambda$.
\end{assumption}
Here the sparsity assumption implies that under the conditioning of the label to either value, the sparsity parameter is bounded. This conditional sparsity can be exploited by an algorithm for learning the conditional distribution whereas a direct regression algorithm may be unable to gain from the same.

\begin{remark}
Observe that the generative model of $X$ itself is not sparse since $Y$ is connected to all hidden nodes however conditioned on knowing the label $Y$, the model is now sparse. This assumption is more reasonable than assuming sparsity directly on the model of $X$ which may not hold.
\end{remark}

\begin{assumption}[Balanced Label]\label{ass:srbm3} For $y \in \{\pm 1\}$, $\Pr[Y = y] \ge \beta$.
\end{assumption}
The above assumption essentially rules out trivial constant learners. Using data, it is easy to check if this assumption is satisfied or not.

As before, we can compute the conditional mean function of the label as follows:
\[ \E[Y | X = x] = \tanh\left(b^{(3)} + \sum_{j} \tanh^{-1}\left(\tanh(w_j) \nu_{j}\right)\right) \]
where $\nu_{j} := \tanh\left(b^{(2)}_{j} + \sum_i \tanh^{-1}\left(\tanh(W_{ij}) X_i\right)\right) = \tanh\left(b^{(2)}_{j} + \sum_i W_{ij} X_i\right).$ This represents a 2-layer neural network and in the limit of infinite hidden nodes, it can represent all 2-layer $\tanh$ networks (see Lemma \ref{lem:express}). 
\begin{assumption}[Boundedness] \label{ass:srbm4}
When $\E[Y|X = x]$ is re-expressed as $\tanh(f^*(x) + b^*)$ for some function $f^*$ with no constant term and $b^* \in \mathbb{R}$. $|b^*| \le B$ for some $B> 0$.
\end{assumption}
The above assumption intuitively says that the effect on $Y$ that does not depend on $X$ is bounded. $B$ can be bounded in terms of the network parameters.

Also observe that conditioned on a fixed label,
\[
\Pr[X= x, H=h | Y= y] \propto \exp\left(\langle x, Wh \rangle + \langle b^{(1)}, x \rangle + \langle b^{(2)} + w y, h \rangle \right) 
\]
which is a sparse, ferromagnetic RBM with arbitrary external field. Thus, we capture a neural network problem with a conditional RBM distributional assumption on the input. This distributional assumption seems more natural than the Gaussian input distribution which is extensively used in prior work. Also, this assumption allows us to leverage prior known algorithms for structure learning of ferromagnetic RBMs to learn the prediction function.

\subsection{Preliminaries: Structure Learning of RBMs with Ferromagnetic Interactions}
Consider a RBM with the following additional assumptions:

\begin{assumption}[Minimum Ferromagnetic Interaction] For all $i \in [n_1] ,j \in [n_2]$ either $W_{ij} = 0$ or $W_{ij} \ge \alpha$.
\end{assumption}

\begin{assumption}[Sparsity] For all $i \in [n_1]$, $\sum_{j = 1}^{n_2} W_{ij} + |b^{(1)}_i| \leq \lambda$ and for all $j \in [n_2]$, $\sum_{i = 1}^{n_1} W_{ij} + |b^{(2)}_j| \leq \lambda$.
\end{assumption}

Under these assumptions, \cite{goel2019learning} has shown that a simple greedy algorithm based on covariance maximization suffices to learn the structure of the RBM. Under the further assumption of non-negative external fields, \cite{bresler2019learning} previously showed a similar greedy maximization algorithm with better dependence on the sparsity parameter $\lambda$.

The crucial structural property that \cite{goel2019learning} use is their algorithm is the following strengthening of the FKG inequality,
\begin{lemma}[Lemma 2 of \cite{goel2019learning}]\label{lem:cond}
For any observed nodes $u, v$ and set $S \subseteq [n_1] \backslash \{u,v\}$,
\[
\mathsf{Cov}(u, v| X_S = x_S):= \E[X_uX_v| X_S = x_S] - \E[X_u|X_S = x_S]~\E[X_v | X_S =x_S] \ge \alpha^2\exp(-12\lambda).
\]
\end{lemma}
Subsequently they define {\em average conditional covariance} $\Cov^\avg(u,v|S) = \E_{x_S}[\cov(u,v|X_S = x_S)]$ which straightforwardly is lower bounded by an application of the above lemma. Their final algorithm essentially greedily maximizes this average conditional covariance to build the neighborhood.

   

\begin{theorem}[Theorem 2 of \cite{goel2019learning}]\label{thm:main}
Consider $M$ samples $\mathcal{S}$ drawn from a RBM with arbitrary external field satisfying the given assumptions. For $\tau = \frac{\alpha^2}{2}\exp(-12 \lambda)$ and $\delta = \exp(-2\lambda)/2$, with probability $1 - \zeta$, $\textsc{LearnRBMNbhd}(u, \tau, \mathcal{S})$ outputs {\em exactly} the two-hop neighborhood of observed variable $u$ for
\begin{align*}
M \ge \Omega\left(\left(\log(1/\zeta) + T^*\log(n)\right)\frac{2^{2T^*}}{\tau^2 \delta^{2T^*}}\right)\text{ and } T^* = \frac{8}{\tau^2}.
\end{align*}
Moreover, the algorithm runs in time $O(T^* M n)$.
\end{theorem}


\subsection{Prediction from Distribution Learning}
Here we will present our algorithm for learning the supervised RBM followed by a proof of its correctness. Instead of learning the label function directly, we will instead first learn the underlying generative model of $X$ conditioned on a particular value of the label and use this knowledge to predict $Y$. 

\begin{theorem} \label{thm:srbm_main}
Given a supervised RBM satisfying Assumption \ref{ass:srbm1}, \ref{ass:srbm2}, \ref{ass:srbm3} and \ref{ass:srbm4}, there exists an algorithm with sample complexity $m = n^2\exp(\lambda)^{\exp(O(\lambda))}(1/\alpha)^{O(1)} (1/\beta)^{O(1)} \log(n/\delta)/\epsilon^2$ and runtime $poly(m)$ returns hypothesis $h$ such that,
\[
\E[\ell(h(X), Y] - \E[\ell(h^*(X), Y] \le \epsilon
\]
where $\ell$ is the logistic loss and $h^*$ is the minimizer of the logistic loss. 
\end{theorem}
\begin{remark}\label{rmk:ferromagnetic-advantage}
For an example where this algorithm is better than if we have no distributional assumptions, observe that we can construct a ferromagnetic RBM where $\E[Y | X]$ is a sparse parity function by adapting in a straightforward way the reduction used in the proof of the part of Theorem~\ref{thm:net-hardness} with bounded $\lambda$ (the use of $\tanh$ as opposed to $f_{\beta}$ in that construction is not fundamental, or we can use a finite version of Lemma~\ref{lem:express}), since the hidden units in that proof all have nonnegative weights. It's clear why Algorithm~\textsc{LearnSupervisedRBMNBhd} is better than an algorithm which doesn't know the input distribution: under the true input distribution, the visible units involved in the parity are correlated so the algorithm can find them, which makes learning the sparse parity easy.
\end{remark}


Our main algorithm can be broken down into three main steps: 1) Use greedy maximization (similar to Algorithm 1 of \cite{goel2019learning}) to first learn the two-hop neighborhood $\mathcal{N}(i)$ of each observed variable $i$ w.r.t. the hidden layer conditioned on the label, 2) For each observed variable $X_i$, learn the distribution for  $X|Y = y$ for $y = \pm 1$, and 3) Use the estimated distribution to compute $\E[Y|X]$. 

\paragraph{Structure Learning}
For notation simplicity, we will overload notation and represent $\cov^\avg(u,v|S, Y) = \E_{x_S, y}[\cov(u,v|X_S = x_S, Y=y)]$ where $\cov(u,v|X_S = x_S, Y=y) = \E[X_uX_v| X_S = x_S, Y = y] - \E[X_u|X_S = x_S, Y = y]~\E[X_v | X_S =x_S, Y = y]$. Then for structure learning, our algorithm essentially follows Algorithm 1 of \cite{goel2019learning} with the slight modification of conditioning w.r.t. $Y$.

\begin{theorem}\label{thm:srbm1}
Consider $m$ samples $\mathcal{S}$ drawn from a supervised RBM satisfying Assumption \ref{ass:srbm1}, \ref{ass:srbm2} and \ref{ass:srbm3}. For $\tau = \frac{\beta\alpha^2}{2}\exp(-12 \lambda)$ and $\delta = \exp(-2\lambda)/2$, with probability $1 - \zeta$, \\$\textsc{LearnSupervisedRBMNbhd}(u, \tau, \mathcal{S})$ outputs {\em exactly} the two-hop neighbors of observed variable $u$ w.r.t. the hidden layer, with
\begin{align*}
m \ge \Omega\left(\left(\log(1/\zeta) + T^*\log(n)\right)\frac{2^{2T^*}}{\tau^2\beta \delta^{2T^*}}\right)\text{ and } T^* = \frac{8}{\tau^2}.
\end{align*}
Moreover, the algorithm runs in time $O(T^* M n)$.
\end{theorem}
\begin{proof}
In order to apply Theorem \ref{thm:main} to our setting, the only two properties we need to show are 1) given the conditioning of $Y$, the average conditional covariance bound still holds, that is, $\cov^\avg(u, v| S \cup \{0\})$ is lower bounded for all $S \subseteq [n_2] \backslash \{u,v\}$ for $v$ in the two-hop neighborhood of $u$, 2) $\Pr[X_S = x_S, Y= y]$ for all $x_S$ and $y$. We have,
\begin{align*}
\cov^\avg(u, v| S, Y) &= \sum_{y \in \pm 1} \sum_{x_S \in \{\pm 1\}^{|S|}} \Pr[X_S = x_S, Y = y] \cov(u, v| X_S = x_S, Y= y)
\end{align*}
By Assumption \ref{ass:srbm2}, we know that either for $y = 1$ or $y = -1$ (say $y = 1$ WLOG), the resulting RBM is sparse therefore we can apply Lemma \ref{lem:cond} to the ones conditioned on $y = 1$. Also, we know that $\cov(u,v| X_S = x_S, Y= y) \ge 0$ for all $x_S$ and $y$ due to FKG inequality for ferromagnetic RBMs. This implies that,
\begin{align*}
\cov^\avg(u, v| S, Y) &\ge \sum_{x_S \in \{\pm 1\}^{|S|}} \Pr[X_S = x_S, Y = 1] \cov(u, v| X_S = x_S, Y= 1)\\
&\ge \sum_{x_S \in \{\pm 1\}^{|S|}} \Pr[X_S = x_S, Y = 1] \alpha^2 \exp(-12 \lambda)\\
&\ge \Pr[Y = 1] \alpha^2 \exp(-12 \lambda) \ge \beta \alpha^2 \exp(-12 \lambda).
\end{align*}
For the second part, let us order the elements of $S$ of size $k$ as $s_1, \ldots, s_{k}$, then we have
\begin{align*}
    \Pr[X_S = x_S, Y= y] &= \Pr[Y= y] \times \Pr[X_{s_1} = x_{s_1}| Y = y]\times \Pr[X_{s_2} = x_{s_2}| X_{s_1} = x_{s_1}, Y = y] \times\ldots \\
    &\quad \times\Pr[X_{s_{k}} = x_{s_{k}}| X_{s_1} = x_{s_1}, \ldots, X_{s_{k - 1}} = x_{s_{k - 1}}, Y = y]
\end{align*}
Since $l_1$-norm to the observed nodes is bounded by $\lambda$, by Bresler's property (see \cite{Bresler}) we have $\Pr[X_{s_r} = x_{s_r}|X_{s_1} = x_{s_1}, \ldots, X_{s_{r}} = x_{s_{r}}, Y = y] \ge \delta$. This implies that $\Pr[X_S = x_S, Y= y] \ge \beta \delta^{|S|}$ for all values of $x_S$ and $y$. Now by applying Theorem \ref{thm:main} with the correct parameters, we get the required result.
\end{proof}

\paragraph{Distribution Learning} Given the neighborhood of each observed node, we run Algorithm \textsc{DistributionFromStructure} and subsequently use Lemma \ref{lemma:dist-recovery-known} to guarantee that we obtin the weights of the unnormalized MRFs for distributions $X|Y = y$ for $y \in \{\pm 1\}$ up to epsilon accuracy. More formally,
\begin{lemma}
Let the maximum two-hop degree of any visible node is at most $d_2$ and $\|b^{(1)}\|_{\infty} \le B$. For $\delta > 0$ and $m = \Omega\left(n^2 \left(\frac{2}{(1 - \tanh(\lambda))}\right)^{d_2 + 1}\log(n/\delta)/\epsilon^2\right)$ we have that with probability at least $1 - \delta$, Algorithm~\textsc{DistributionFromStructure} given $m$ samples and $\mathcal{\widehat{N}}(i) = \mathcal{N}(i)$ for every $i$ returns unnormalized MRFs of $X|Y = y$ for $y\in \{\pm 1\}$ with coefficients $\hat{f}^{(y)}_S$  that are close to the coefficients of the true unnormalized MRFs $f^{(y)}_S$, that is, 
\[\sum_S |\hat{f}^{(y)}_S - f^{(y)}_S| \le \epsilon. \]
\end{lemma}

\paragraph{Constructing the Predictor} 
Observe that the joint distribution of $X$ and $Y$ can be represented as,
\[
\Pr[X = x, Y= y] \propto \exp\left(\sum_S f^{(1)}_S x_S \mathbbm{1}[y = 1] + \sum_S f^{(-1)}_S x_S \mathbbm{1}[y = -1] + b^* y\right)
\]
for some $b^*$ and coefficients of the true unnormalized MRFs $f^{(y)}_S$ corresponding to conditioning of $Y= y$. This gives us,
\[
\E[Y|X = x] = \tanh\left(\sum_S \frac{(f^{(1)}_S - f^{(-1)}_S)}{2} x_S + b\right) \approx_\eps \tanh\left(\sum_S \frac{(\hat{f}^{(1)}_S - \hat{f}^{(-1)}_S)}{2} x_S + b\right)
\]
Since we have estimates of $f^{(y)}_S$, to learn the predictor for $Y$ we only need to find $b^*$ which we can find by minimizing $\ell$ snce it is convex. Let  $h_b = \sum_S \frac{(f^{(1)}_S - f^{(-1)}_S)}{2} x_S + b$ and $\hat{h}_b = \sum_S \frac{(\hat{f}^{(1)}_S - \hat{f}^{(-1)}_S)}{2} x_S + b$. We minimize $\hat{E}[\ell(h_b(X), Y)]$ over $b$ and suppose the minimizer is $\hat{b}$. By Fact \ref{fact:logistic-loss}.3,  $\ell(\hat{h}_b(X), Y) \le \ell(h_b(X), Y) + 4\epsilon$. By Fact \ref{fact:logistic-loss}.4, $h_{b^*}$ is the minimizer of the logistic loss.  Then we have,
\[
\hat{\E}[\ell(h_{b}(X), Y)] \le \hat{\E}[\ell(\hat{h}_{b^*}(X), Y)] + 4 \epsilon \le \hat{\E}[\ell(h_{b^*}(X), Y)] + 8 \epsilon.
\]
Last we need a generalization bound that holds for our hypothesis class. For this we bound the Rademacher complexity (see \cite{shalev2014understanding} for more background) of the class of functions $\ell \circ \mathcal{H}$ where $\mathcal{H}:= \{h_b||b| \le B\}$.
\begin{align*}
    \mathcal{R}_m(\ell \circ \mathcal{H}) &\le 2 \mathcal{R}_m(\mathcal{H}) \\
    &= \E_\sigma \left[\sum_{b | |b| \le B} \frac{1}{m} \sum_{i=1}^m \sigma_i h_b(x^{(i)})\right]\\
    &= \E_\sigma \left[\sum_{b | |b| \le B} \frac{1}{m} \sum_{i=1}^m \sigma_i \sum_S (f^{(1)}_S - f^{(-1)}_S) x_S + 2b\right]\\
    &= 2\E_\sigma \left[\sum_{b | |b| \le B} \frac{1}{m} \sum_{i=1}^m \sigma_ib\right]\\
    & = 2 B \E_\sigma \left[ \frac{1}{m}\left| \sum_{i=1}^m \sigma_i \right|\right]\\
    &\le \frac{2B}{\sqrt{m}}.
\end{align*}
Here the first inequality follows from the contraction lemma (see \cite{ledoux2013probability}) and the last from standard properties of Radmeacher variables. Now applying Theorem 26.5 from \cite{shalev2014understanding} we get 
\[
|\E[\ell(h_{b}(X), Y)] - \E[\ell(\hat{h}_{b}(X), Y)]| \le \frac{2B}{\sqrt{m}} + c \sqrt{\frac{\log(1/\delta)}{\sqrt{m}}}
\]
where $c$ is the maximum value of logistic loss by any hypothesis in the class. Observe that by Fact \ref{fact:logistic-loss}.4, logistic loss at $h_{b^*}$ is bounded by a constant. Hence by Lipschitzness, we know that loss anywhere will be bounded by $O(\max(1, B))$. Therefore choosing $m \ge \Omega(B^2 \log(1/\delta)/\epsilon^2)$ suffices to get within $\epsilon$. Combining this with before we get that the loss is within $O(\eps)$ of the best loss.

\paragraph{Proof of Theorem \ref{thm:srbm_main}} First, the algorithm runs \textsc{LearnSupervisedRBMMbhd} for each node to learn the structure of the induced RBM exactly with the given samples \[m_1 = \exp(\lambda)^{\exp(O(\lambda))}(1/\alpha)^{O(1)} (1/\beta)^{O(1)} \log(n/\delta).\] With the structure, we run \textsc{DistributionFromStructure} to learn both the induced RBMs for each conditioning of the label using $m_2 \ge \Omega\left(n^2 \left(\frac{2}{(1 - \tanh(\lambda))}\right)^{d_2 + 1}\log(n/\delta)/\epsilon^2\right)$ samples where $d_2$ is the max 2-hop neighborhood size. Note that the dependence on $\lambda$ is greater in $m_1$ than $m_2$. Subsequently, given the unnormalized mrfs, we run a simple optimization to find the bias term of the predictor using $m_3 \ge \Omega(B^2 \log(1/\delta)/\epsilon^2)$ samples. Combining the learnt mrf and the bias term, we get our hypothesis.

\begin{remark}\label{rmk:same-but-no-ferromagnetic}
If the model is not ferromagnetic, it is also possible and we expect it may be advantageous in some models to still use a similar indirect approach based on Bayes rule for learning a predictor of $Y$, but using the result of Theorem~\ref{thm:node-prediction} instead of the greedy structure recovery method used in this section. The disadvantage of this approach is of course that its runtime for achieving structure recovery is slower.
\end{remark}

\section{Additional Experimental Data}\label{apdx:reference-images}
Figure~\ref{fig:mnist_reference} contains samples generated from the model trained on MNIST images.
\begin{figure}
    \centering
    \includegraphics{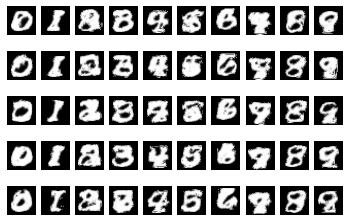}
    \caption{Five i.i.d. samples for each MNIST class, drawn from the trained model by Gibbs sampling.}
    \label{fig:mnist}
\end{figure}
For reference, we also include samples from the true MNIST and FashionMNIST training sets in the same format as Figure~\ref{fig:mnist} and Figure~\ref{fig:fashionmnist}.
\begin{figure}
    \centering
    \includegraphics{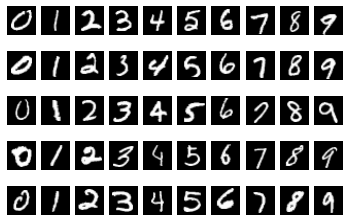}
    \caption{Reference MNIST images chosen randomly from training set.}
    \label{fig:mnist_reference}
\end{figure}
\begin{figure}
    \centering
    \includegraphics{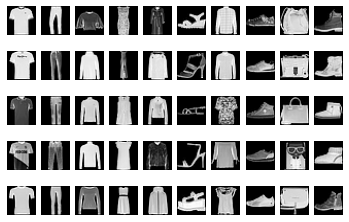}
    \caption{Reference FashionMNIST samples from training set.}
    \label{fig:fashionmnist_reference}
\end{figure}
\end{document}